\documentclass[11pt,a4paper,onecolumn,gdm]{google}
\usepackage[authoryear,sort&compress,round]{natbib}
\usepackage{placeins}
\usepackage{etoolbox}
\patchcmd{\maketitle}
  {\begin{adjustwidth}{0pt}{24pt}}
  {\begin{adjustwidth}{0pt}{0pt}}
  {}{\PackageError{VideoEncoders}{Could not widen the title block}{Check the google.cls title definition.}}
\makeatletter
\def\@textbottom{\vskip 0pt plus 4pt}
\makeatother
\renewcommand{\today}{2026-10-05}
\uselogo{}
\usepackage{amsthm}

\usepackage{amsmath,amsfonts,bm}

\def\eqref#1{equation~\ref{#1}}

\def\1{\bm{1}}

\DeclareMathAlphabet{\mathsfit}{\encodingdefault}{\sfdefault}{m}{sl}
\SetMathAlphabet{\mathsfit}{bold}{\encodingdefault}{\sfdefault}{bx}{n}

\usepackage{graphicx,wrapfig,booktabs,makecell,multirow,adjustbox}
\usepackage{array,tabularx,colortbl}
\usepackage{pifont,enumitem}
\usepackage{tcolorbox}
\tcbuselibrary{breakable}

\usepackage{placeins}
\definecolor{cvprblue}{rgb}{0.21,0.49,0.74}
\usepackage{hyperref}
\definecolor{IFRepresentation}{HTML}{1969AA}
\definecolor{IFAllocation}{HTML}{B96811}
\definecolor{IFInteraction}{HTML}{008579}
\definecolor{IFTheory}{HTML}{7050A0}
\newcommand{\IFrep}[1]{#1}
\newcommand{\IFalloc}[1]{#1}
\newcommand{\IFinter}[1]{#1}

\hypersetup{colorlinks=true,linkcolor=black,citecolor=cvprblue,
  urlcolor=cvprblue,pdftitle={Image-First Video Encoding}}

\newtheorem{IFproposition}{Proposition}

\makeatletter
\newcounter{algorithm}
\renewcommand{\thealgorithm}{\arabic{algorithm}}

\newcommand{\fps@algorithm}{tbp}
\newcommand{\ftype@algorithm}{4}
\newcommand{\ext@algorithm}{loa}
\newcommand{\fnum@algorithm}{Algorithm~\thealgorithm}
\newenvironment{algorithm}[1][tbp]{\@float{algorithm}[#1]}{\end@float}
\newenvironment{algorithm*}[1][tbp]{\@dblfloat{algorithm}[#1]}{\end@dblfloat}
\makeatother
\definecolor{IFTable}{RGB}{255,239,214}
\definecolor{IFReference}{RGB}{233,246,232}
\definecolor{IFProtocol}{RGB}{240,235,251}
\definecolor{IFContext}{RGB}{229,244,241}
\definecolor{IFBoxFill}{RGB}{246,245,253}
\definecolor{IFBoxEdge}{RGB}{78,76,190}
\definecolor{lightroyalblue}{HTML}{F6F8FD}
\definecolor{linkblue}{rgb}{0,0.08,0.45}
\newcommand{\oursbest}[1]{\cellcolor{IFTable}\textbf{#1}}
\definecolor{IFYes}{RGB}{32,137,76}
\definecolor{IFNo}{RGB}{184,48,48}
\newcommand{\cmark}{\ding{51}}
\newcommand{\xmark}{\ding{55}}

\newtcolorbox{paperbox}{colback=IFBoxFill,colframe=IFBoxEdge,
 boxrule=0.45pt,arc=2.5pt,left=5.5pt,right=5.5pt,top=4pt,bottom=4pt,
 before skip=6pt,after skip=6pt}

\newcommand{\proofspacing}{%
  \setlength{\abovedisplayskip}{3pt plus 1pt minus 1pt}%
  \setlength{\belowdisplayskip}{3pt plus 1pt minus 1pt}%
  \setlength{\abovedisplayshortskip}{2pt plus 1pt}%
  \setlength{\belowdisplayshortskip}{3pt plus 1pt minus 1pt}%
  \allowdisplaybreaks[3]}
\usepackage{refcount}
\definecolor{IFControl}{RGB}{227,239,240}
\definecolor{IFPairTint}{RGB}{249,244,245}
\newcount\suppendpage
\newcommand{\supppages}[2]{\begingroup
  \suppendpage=\getpagerefnumber{#1}\relax
  \def\keepend##1{\ifnum\getpagerefnumber{##1}>\suppendpage\relax
    \suppendpage=\getpagerefnumber{##1}\relax\fi}%
  \forcsvlist{\keepend}{#2}%
  \pageref{#1}\ifnum\suppendpage>\getpagerefnumber{#1}\relax
    --\hyperlink{page.\the\suppendpage}{\the\suppendpage}\fi\endgroup}

\newlength{\paperplotwidth}
\newenvironment{figuretext}[1]{%
  \par\addvspace{.5\baselineskip}\noindent
  \setlength{\paperplotwidth}{#1}%
  \begin{minipage}[t]{\dimexpr\linewidth-\paperplotwidth-1.35em\relax}%
  \vspace{0pt}\raggedright\setlength{\parskip}{.35\baselineskip}\ignorespaces
}{%
  \end{minipage}\par\addvspace{.5\baselineskip}%
}
\newcommand{\besidefigure}{%
  \end{minipage}\hfill
  \begin{minipage}[t]{\paperplotwidth}%
  \vspace{0pt}\captionsetup{hypcap=false}\ignorespaces
}

\definecolor{IFTable}{RGB}{242,242,242}
\definecolor{IFReference}{RGB}{250,250,250}
\definecolor{IFProtocol}{RGB}{247,247,247}
\definecolor{IFContext}{RGB}{245,245,245}

\title{Image-First Video Encoding}

\author[2,*]{Jusheng Zhang}
\author[3,*,$\dagger$]{Wenhao Wang}
\author[1]{Longqi Cai}
\author[1]{Liangzhe Yuan}
\author[1]{Yuxiao Wang}
\author[1,$\dagger$]{Ming-Hsuan Yang}
\affil[1]{Google DeepMind}
\affil[2]{Stanford University}
\affil[3]{Vast Intelligence Lab}
\affil[*]{Equal contribution}
\affil[$\dagger$]{Corresponding authors}
\correspondingauthor{Wenhao Wang (\href{mailto:wangwenhao@vastilab.com}{wangwenhao@vastilab.com}); Ming-Hsuan Yang (\href{mailto:minghsuan@google.com}{minghsuan@google.com}).\\
Google DeepMind authors are only contributing in an advisory capacity.}
\hypersetup{pdfauthor={Jusheng Zhang, Wenhao Wang, Longqi Cai, Liangzhe Yuan, Yuxiao Wang, Ming-Hsuan Yang}}

\begin{abstract}
We find that a simple encoding order is effective for the evaluated Qwen
video-language models: independently encode frames, then jointly retain
their source tokens. Across 13 benchmarks and three backbones, this
training-free image-first pathway maintains comparable Full Image
aggregate accuracy using 28\%--35\% of its visual tokens.
Framewise candidates remain separately selectable until a question-aware
budget is applied. Frame quotas and within-frame DPP diversity allocate
this budget; a temporal correspondence earns a reward only when both
endpoints survive. Selected vectors remain unmerged, with positions and
intermediate streams gathered by the same source indices.
On Qwen3-VL-8B, the selector improves over FastVID by 2.78 points on a
common Image pool, and over additive endpoint rewards by 1.44 points with
kernels and the graph fixed. Single-H200 E2E latency falls from 516.4 to
314.9 ms under EOS-terminated generation. An optional refiner adds context
at unchanged output length.

\end{abstract}

\begin{document}
\maketitle

\begin{center}
\begin{minipage}{\textwidth}
\centering
\captionsetup{hypcap=false}
\includegraphics[width=\linewidth]{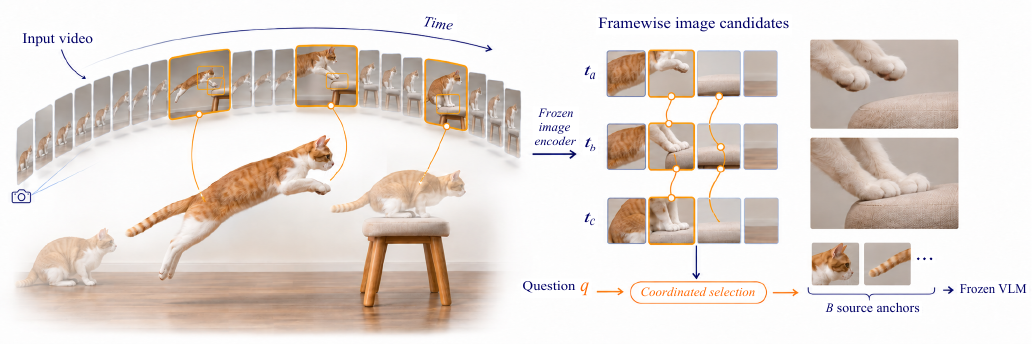}
\captionof{figure}{\textbf{Image-first video encoding.} Frames are encoded independently, then temporal relationships guide joint retention of $B$ source anchors. The montage and paw insets illustrate source access; selected anchors retain their frame and spatial positions.}
\label{fig:detail-teaser}
\end{minipage}
\end{center}

\section{Introduction}
\label{sec:intro}

Separating frame representation from retention yields an effective compact
input for the evaluated Qwen systems:
\emph{encode frames independently; retain source
evidence jointly.} Across 13 benchmarks, this training-free pathway
maintains comparable Full Image aggregate accuracy using 28\%--35\% of
its visual tokens. We call this pathway \emph{image-first video encoding}.

\begin{figure*}[t]
\centering
\includegraphics[width=\linewidth]{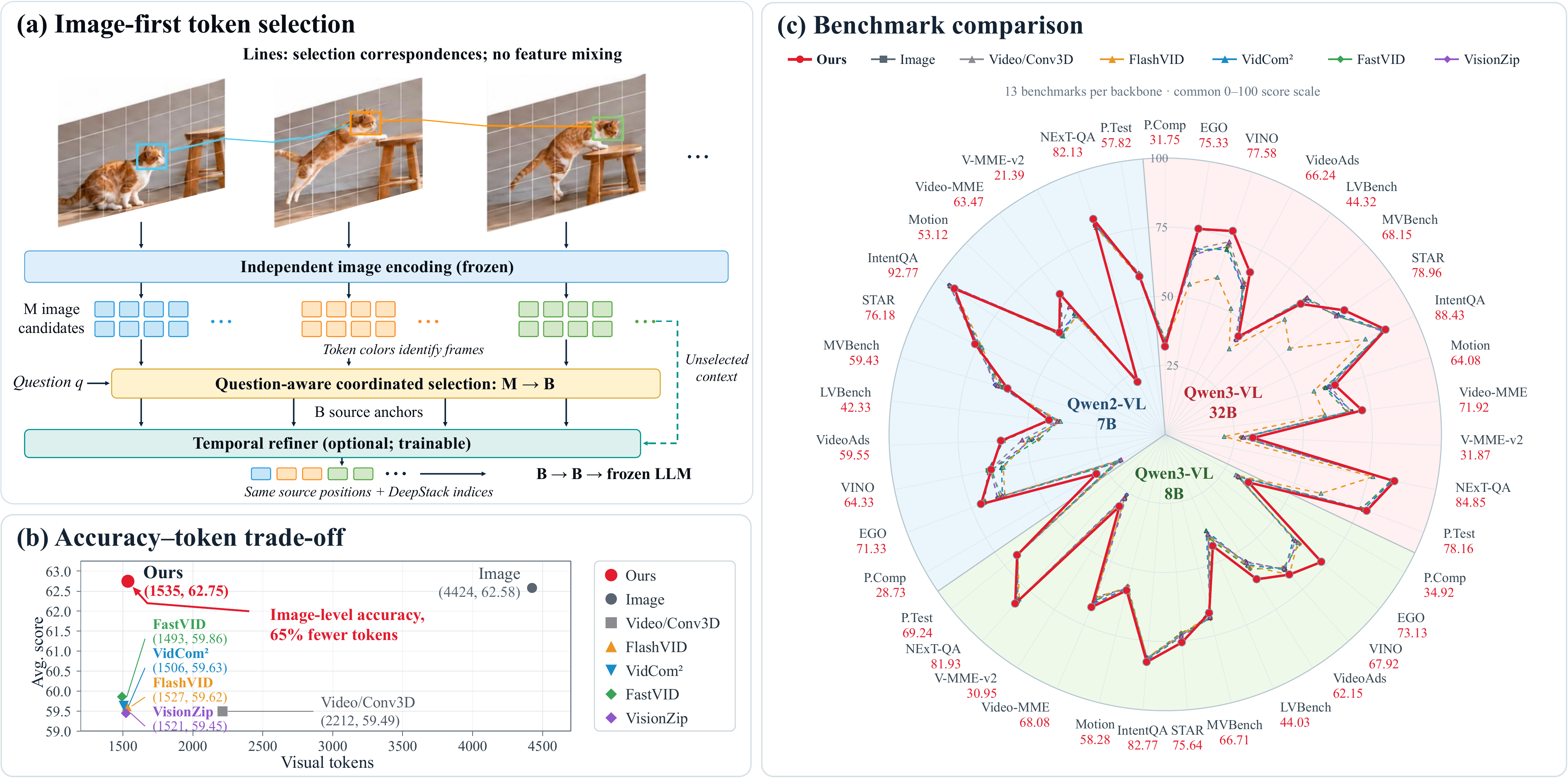}
\caption{\textbf{The image-first pathway.} (a) Question-aware selection retains $B$ source anchors before optional refinement. Complete-pathway results compare (b) accuracy versus tokens on Qwen3-VL-8B and (c) 13 benchmarks across three backbones.}
\label{fig:overview}
\end{figure*}

Native Qwen video pathways embed neighboring frames jointly before
question-aware selection
\citep{wang2024qwen2vl,bai2025qwen25vl,bai2025qwen3vl}.
The resulting candidates no longer expose original per-frame image tokens
as separate choices. Local cues motivate preserving those choices: a gap
versus paw contact distinguishes being airborne from landing
(Fig.~\ref{fig:detail-teaser}). Actual source use requires measurement.

Image pretraining \citep{dosovitskiy2021vit,radford2021clip} already supports
image-derived video understanding
\citep{xu2024slowfastllava,li2025onevision,zhang2025llavavideo};
query-aware selection adapts sampling \citep{zhang2025qframe}.
We study their retention under a compact budget. Full Image forwards every
candidate, using roughly twice the native token count here. Image-first
encoding keeps candidates separately addressable until selection.

Question-aware quotas allocate capacity and within-frame DPP diversity
limits repetition. Cross-frame correspondences guide which observations
survive together: a reward is earned only when both endpoints are retained
(Fig.~\ref{fig:overview}(a)). This couples membership while preserving the
selected embeddings. The same indices gather source positions and
intermediate streams for the frozen language model.

The controls separate two sources of improvement. On Qwen3-VL-8B,
switching FastVID from native to Image candidates adds 1.66 pp; on the
same Image pool, our selector adds 2.78 pp over FastVID.
With kernels and the graph fixed, joint retention adds 1.44 pp over
additive endpoint rewards. Single-H200 E2E latency falls from 516.4 to 314.9 ms
under EOS-terminated generation (Fig.~\ref{fig:source-cost}).

Optional refinement enriches retained anchors with unselected context
at the same output length. Supplementary Secs.~\ref{app:information}--\ref{app:cut}
analyze the information and interaction limits of this interface.
Our contributions are:
\begin{itemize}[leftmargin=*,itemsep=0.36\baselineskip]
\item An empirical finding across three Qwen backbones: independent frame
encoding followed by joint source retention maintains comparable Full
Image aggregate accuracy at a compact budget, without additional training.
\item A realization with frame quotas, within-frame diversity, and an
explicit retained-pair reward, preserving source embeddings, positions,
and intermediate streams for the frozen language model.
\item Candidate-pool and fixed-kernel controls, comparisons with
close baselines in a common pool, and single-H200 measurements that assess the finding's
sources, practical cost, and limits.
\end{itemize}

\section{Problem Formulation}
\label{sec:problem}

We separate candidate construction from the output budget. Given sampled
frames $\mathcal V=(x_1,\ldots,x_T)$, a question $q$, and a feasible budget
$B$, image-first encoding constructs a compact input from framewise candidates.
The frozen image encoder produces $\mathbf Z_t=E_{\mathrm{img}}(x_t)$
independently for each frame. Their concatenation
$\mathbf Z\in\mathbb R^{M\times D}$ contains $M=\sum_tN_t$ candidates,
each with source position $p_i$ and intermediate features $d_i$.
Fig.~\ref{fig:encoding_pathways} contrasts native video, Full Image,
and our pathway. The budget $B$ bounds the language model's visual input;
encoding $M$ candidates still incurs input-dependent cost.

\begin{figure*}[t]
\centering
\includegraphics[width=\linewidth]{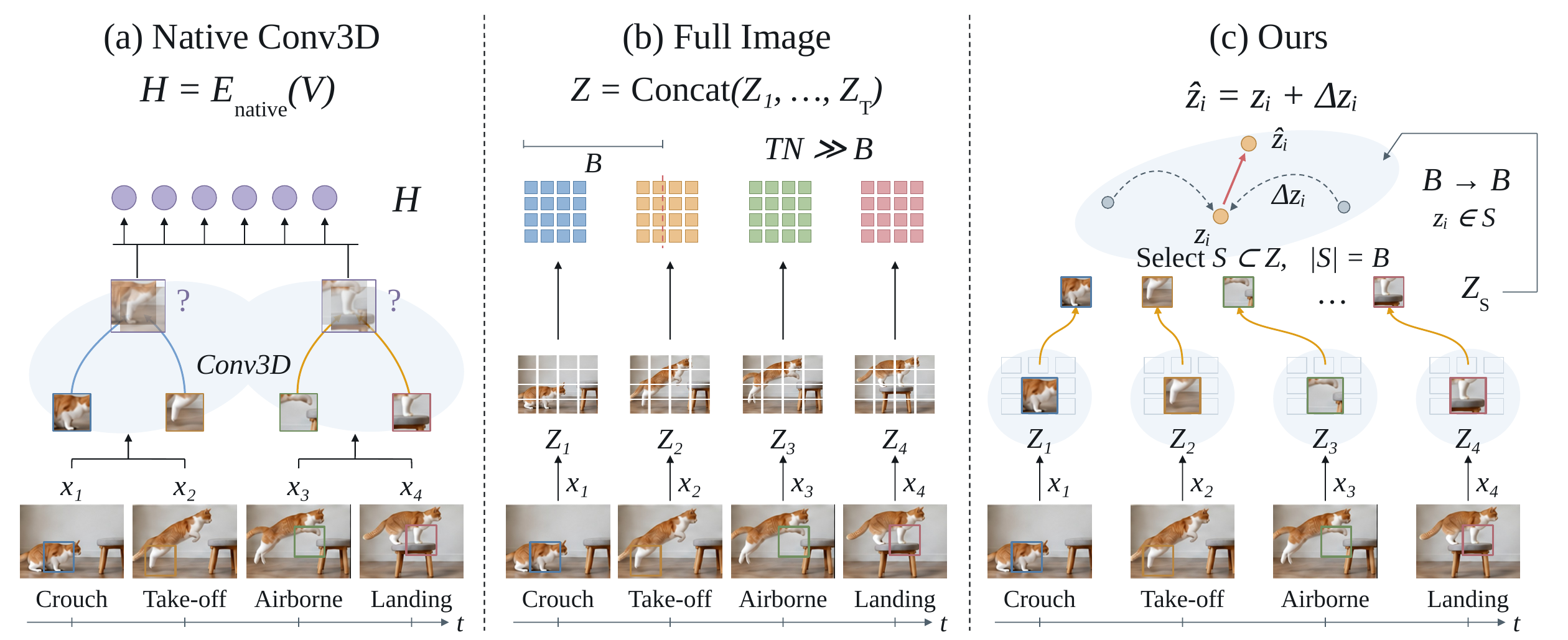}
\caption{\textbf{Three video-input pathways.} Native Conv3D mixes frames early, Full Image forwards all image tokens, and Ours selects $B$ anchors before refinement. ``?'' marks source-frame ambiguity; red denotes budget, not truncation.}
\label{fig:encoding_pathways}
\end{figure*}

\paragraph{A source-aligned interface.}
For $B$ distinct source indices $S$, the same row-selection matrix $P_S$
acts on visual embeddings, positions, and intermediate streams:
\begin{equation}
\begin{aligned}
\mathbf Z_S&=P_S\mathbf Z,\qquad |S|=B,\\
\mathbf P_S&=P_S\mathbf P,\qquad \mathbf D_S=P_S\mathbf D.
\end{aligned}
\label{eq:source_selection}
\end{equation}
Selection preserves source embeddings; refinement updates $\mathbf Z_S$
with $\mathbf P_S$ and $\mathbf D_S$ fixed. Source alignment identifies
an anchor without restricting its refined content to one frame.

\paragraph{Joint retention and subsequent interaction.}
Temporal relationships determine membership in the source-aligned video
record; optional refinement lets its anchors read additional context.
Conditional-prediction and fixed-operator analyses distinguish these
decisions in
Supplementary Secs.~\ref{app:information}--\ref{app:cut}.

\section{Method}
\label{sec:method}

We realize this ordering with frozen framewise features and joint source
retention. Question-aware quotas allocate capacity; within-frame diversity
limits repetition and cross-frame pair utility links retained observations.
Source features remain unmerged, and only $B$ anchors reach the frozen
language model. Context refinement is optional.

\begin{figure*}[t]
\centering
\includegraphics[width=\linewidth]{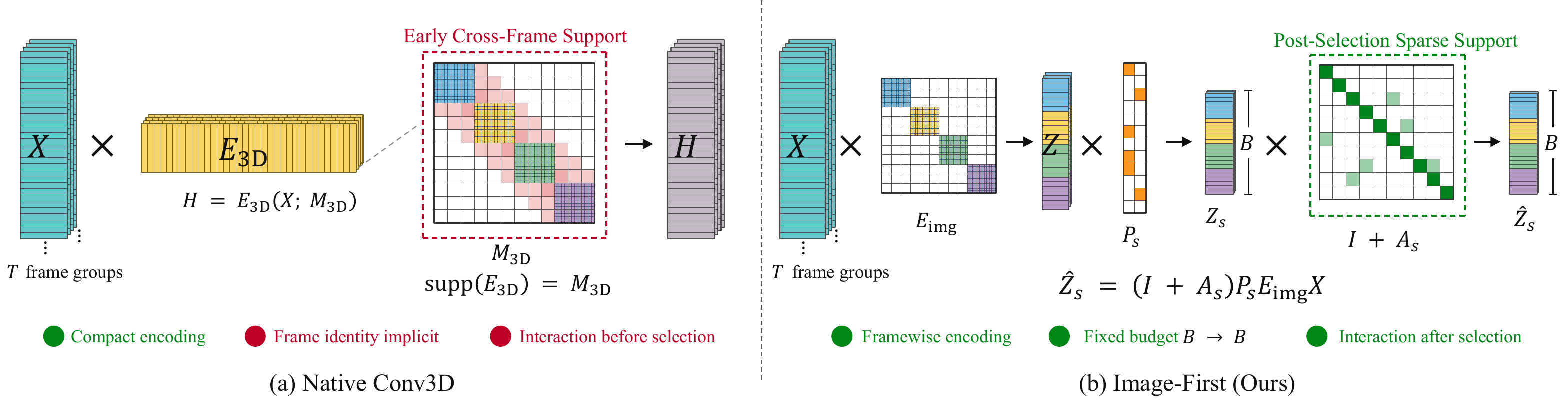}
\caption{\textbf{Selection before temporal refinement: a fixed-operator schematic.} Conv3D couples frames before selection. Image-first encoding selects $B$ source anchors with $P_S$, then applies sparse refinement $(I+A_S)$ without changing their number or positions.}
\label{fig:interaction_order}
\end{figure*}

\subsection{Independent Encoding and Frame Quotas}
\label{sec:independent_representation}

Our training-free allocator assigns integer
quotas from question relevance, motion, saliency, and diversity statistics.
It applies preview and uncertainty weighting, then capacity-aware integer
allocation. Representative source indices determine each frame's quota
$b_t$, with $0\leq b_t\leq N_t$ and $\sum_tb_t=B$.
These quotas are computed once and held fixed during selection;
at the 32-frame, 1,536-token setting, a 12-token frame floor reserves 384
tokens and leaves 1,152 for score-based allocation.
The common floor is $\min\{12,\lfloor B/(2T)\rfloor\}$, reserving at
most half the budget; at 128 frames and $B=1536$ it is six.
Temporal coordination therefore selects evidence within each frame's
assigned capacity, while leaving the allocation of that capacity unchanged.
Supplementary Sec.~\ref{app:selector_quota} details the allocation stages
and common parameters.

Within these quotas, selection chooses original source features and
applies the same indices to every stream,
including MRoPE positions and DeepStack features for Qwen3-VL
\citep{bai2025qwen3vl}. Unselected candidates remain available to the optional
refiner until its context reads finish.

\subsection{Joint Retention of Image Tokens}
\label{sec:coordinated_selection}

Let $S_t=S\cap I_t$ contain the selected indices of frame $t$.
With normalized features $u_i$, positive question-weighted qualities $a_i$,
and stabilizer $\epsilon_t>0$, the within-frame DPP kernel is
\begin{equation}
L_t(i,j)=a_i a_j u_i^\top u_j+\epsilon_t\mathbf1[i=j].
\label{eq:main_dpp_kernel}
\end{equation}
Feature norms and question relevance define $a_i$; normalization and
constants are given in Supplementary Sec.~\ref{app:selector_dpp}.
The log determinant favors useful, non-repeated directions
\citep{kulesza2012determinantal}.

An adjacent-frame graph $E$ links mutual feature matches that pass cosine,
matching-margin, spatial-distance, and scene-consistency filters.
Each edge receives reward $w_{ij}=c_{ij}h_{ij}$ from matching confidence
$c_{ij}$ and local state change $h_{ij}$. The selection objective is
\begin{equation}
\begin{aligned}
J(S)={}&\sum_t\log\det L_t[S_t,S_t]\\
 &+\gamma\sum_{\{i,j\}\in E}w_{ij}\mathbf1[i,j\in S],
\end{aligned}
\label{eq:selection_principle}
\end{equation}
subject to $|S_t|=b_t$. The DPP term limits within-frame redundancy;
the edge term rewards retaining both correspondence endpoints.
A candidate's marginal value thus depends on other retained
observations, without mixing embeddings, source identities, or positions.

\paragraph{An explicit co-selection gain.}
For $U\cap S=\varnothing$, write
$\Delta_U J(S)=J(S\cup U)-J(S)$. When $i$ and $j$ lie in different
frames, the within-frame terms separate and
\begin{equation}
 \Delta_{\{i,j\}}J(S)-\Delta_{\{i\}}J(S)-\Delta_{\{j\}}J(S)
 =\gamma w_{ij},
 \label{eq:retention_complementarity}
\end{equation}
with $w_{ij}=0$ when no edge exists. An additive alternative gives each
selected endpoint $R_i=\frac12\sum_{j:\{i,j\}\in E}w_{ij}$ regardless of
its counterpart; its cross-frame mixed difference is zero.
The factor $1/2$ matches the total edge reward at full retention.
Equivalently, the additive objective includes an extra half-weight reward
for edges with exactly one retained endpoint
(Supplementary Sec.~\ref{app:reward_control}).
Our counterpart-dependent reward encourages, rather than forces, paired
retention. This property does not guarantee temporal reasoning or
super-additive benchmark accuracy. Supplementary Fig.~\ref{fig:retention-toy} and
Supplementary Sec.~\ref{app:retention_toy} compute the distinction on a minimal example.

Greedy selection compares every feasible singleton with feasible
adjacent-frame pairs using gain per added token. A singleton gains its
DPP log-determinant marginal plus rewards to selected neighbors; a pair
also gains its shared edge reward. Singletons remain available whenever a
quota is unfilled, ensuring termination at exactly $B$ distinct anchors.
This is a feasible greedy solver without a global-optimality guarantee.
Supplementary Secs.~\ref{app:selection} and \ref{app:selector_objective} give the
marginals and move rules; Algorithm~\ref{alg:image_first} summarizes inference.

\setcounter{dbltopnumber}{1}
\renewcommand{\dbltopfraction}{0.80}
\begin{table}[t]
\caption{\textbf{Training-free core.} Scores are mean accuracies over 13 benchmarks; TF omits and +Ref. adds the learned refiner. Both retain about 1.5K tokens; the last column gives their fraction of Full Image tokens.}
\label{tab:tf-main}
\centering
\small
\setlength{\tabcolsep}{3pt}
\renewcommand{\arraystretch}{1.10}
\begin{tabularx}{\linewidth}{@{}l*{4}{>{\centering\arraybackslash}X}@{}}
\toprule
Backbone & Image & TF & +Ref. & Tokens \\
\midrule
Qwen2 / 7B & 59.26 & 59.37 & 59.43 & 27.6\% \\
Qwen3 / 8B & 62.58 & 62.64 & 62.75 & 34.7\% \\
Qwen3 / 32B & 66.09 & 66.22 & 66.28 & 34.7\% \\
\bottomrule
\end{tabularx}
\end{table}

\begin{table}[!t]
\caption{\textbf{Six-benchmark accuracy.} Scores (\%) average three seeds; Avg. is the six-task mean. MB/MME/M2/NQ/PT/PC denote MotionBench, Video-MME, Video-MME-v2, NExT-QA, Perception Test/Comp. All except +ref. are training-free; +ref. adds refinement. Bold/underlined scores rank budget-matched methods.}
\label{tab:part-i}
\centering
\begingroup
\fontsize{9.5}{11.4}\selectfont
\setlength{\tabcolsep}{0.8pt}
\renewcommand{\arraystretch}{1.05}
\setlength{\aboverulesep}{1pt}
\setlength{\belowrulesep}{1pt}

\begin{tabular*}{\linewidth}{@{\extracolsep{\fill}}l*{8}{c}@{}}
\toprule
Method & Tokens & MB & MME & M2 & NQ & PT & PC & Avg. \\
\midrule

\rowcolor{IFProtocol}
\multicolumn{9}{l}{\textit{Qwen2-VL-7B}} \\[-1pt]

\rowcolor{IFReference}
Video/Conv3D
& 2,780
& 53.61
& 57.78
& 22.50
& 80.33
& 59.17
& 18.39
& 48.63 \\

\rowcolor{IFReference}
Image
& 5,561
& 51.62
& 62.78
& 22.03
& 81.55
& 58.65
& 28.25
& 50.81 \\[1pt]

Uniform
& 1,536
& 50.11
& 51.65
& 20.35
& 76.88
& 55.42
& 17.65
& 45.34 \\

FlashVID
& 1,536
& 51.27
& 53.64
& 21.67
& 78.89
& 57.39
& 19.13
& 47.00 \\

VidCom$^{2}$
& 1,529
& 51.84
& \underline{54.78}
& 22.37
& 79.77
& \underline{58.02}
& 19.24
& 47.67 \\

FastVID
& 1,522
& 52.02
& 53.49
& \textbf{22.84}
& \underline{80.24}
& 57.79
& \underline{21.13}
& 47.92 \\

VisionZip
& 1,532
& \textbf{53.26}
& 54.03
& \underline{22.70}
& 79.53
& \textbf{58.34}
& 21.06
& \underline{48.15} \\

\rowcolor{IFTable}
\textbf{Ours + ref.}
& 1,536
& \underline{53.12}
& \oursbest{63.47}
& 21.39
& \oursbest{82.13}
& 57.82
& \oursbest{28.73}
& \oursbest{51.11} \\

\midrule

\rowcolor{IFProtocol}
\multicolumn{9}{l}{\textit{Qwen3-VL-8B}} \\[-1pt]

\rowcolor{IFReference}
Video/Conv3D
& 2,212
& 57.34
& 65.00
& 27.50
& 80.91
& 68.99
& 30.49
& 55.04 \\

\rowcolor{IFReference}
Image
& 4,424
& 58.08
& 67.96
& 30.78
& 81.84
& 69.28
& 34.53
& 57.08 \\

Uniform
& 1,536
& 55.48
& 61.22
& 24.35
& 77.12
& 66.85
& 27.65
& 52.11 \\

FastV
& 1,536
& 56.12
& 63.45
& 25.88
& 78.45
& 67.24
& 29.15
& 53.38 \\

SparseVLM
& 1,536
& 56.45
& 63.85
& 26.15
& 78.95
& 67.65
& 29.55
& 53.77 \\

DyCoke
& 1,536
& 56.85
& 64.25
& 26.65
& 79.35
& 68.12
& 29.85
& 54.18 \\

FlashVID
& 1,527
& 57.12
& 64.40
& 27.69
& 79.75
& 68.85
& 29.58
& 54.57 \\

VidCom$^{2}$
& 1,506
& 57.31
& \underline{66.88}
& 27.22
& 80.51
& 68.94
& 30.98
& 55.31 \\

FastVID
& 1,493
& 56.63
& 65.69
& 27.23
& 81.08
& 68.72
& 30.07
& 54.90 \\

VisionZip
& 1,521
& 57.43
& 64.12
& 26.05
& 81.18
& \underline{69.11}
& 30.08
& 54.66 \\

\rowcolor{IFTable}
\textbf{Ours (TF)}$^\dagger$
& 1,536
& \textbf{58.98}
& 65.40
& \textbf{31.45}
& \textbf{82.33}
& 68.65
& \textbf{35.25}
& \underline{57.01} \\

\rowcolor{IFTable}
\textbf{Ours + ref.}
& 1,535
& \underline{58.28}
& \textbf{68.08}
& \underline{30.95}
& \underline{81.93}
& \textbf{69.24}
& \underline{34.92}
& \textbf{57.23} \\

\midrule

\rowcolor{IFProtocol}
\multicolumn{9}{l}{\textit{Qwen3-VL-32B}} \\[-1pt]

\rowcolor{IFReference}
Video/Conv3D
& 2,212
& 62.19
& 67.96
& 28.44
& 84.53
& 77.53
& 33.63
& 59.05 \\

\rowcolor{IFReference}
Image
& 4,424
& 63.93
& 71.85
& 31.72
& 84.76
& 78.00
& 31.39
& 60.28 \\[1pt]

Uniform
& 1,536
& 58.25
& 62.15
& 24.85
& 80.12
& 72.55
& 30.15
& 54.68 \\

FlashVID
& 1,527
& 56.14
& 58.29
& 21.46
& 76.89
& 60.42
& 32.71
& 50.99 \\

VidCom$^{2}$
& 1,504
& 60.52
& \underline{68.38}
& \underline{28.23}
& 84.01
& \underline{76.92}
& \underline{33.68}
& \underline{58.62} \\

FastVID
& 1,485
& 60.12
& 66.42
& 27.55
& 83.59
& 76.12
& \textbf{34.48}
& 58.05 \\

VisionZip
& 1,521
& \underline{61.64}
& 66.54
& 27.61
& \underline{84.12}
& 75.89
& 33.22
& 58.17 \\

\rowcolor{IFTable}
\textbf{Ours + ref.}
& 1,535
& \oursbest{64.08}
& \oursbest{71.92}
& \oursbest{31.87}
& \oursbest{84.85}
& \oursbest{78.16}
& 31.75
& \oursbest{60.44} \\

\bottomrule
\end{tabular*}

\endgroup
\end{table}

\subsection{Optional Temporal Refinement}
\label{sec:temporal_refinement}
\label{sec:learning_objective}

In the optional extension, each anchor $i$ retrieves $K=8$ feature-similar tokens from nearby
frames of the unchanged image pool; $C(i)$ need not lie in $S$.
Bottleneck attention uses relative temporal and spatial features to form
weights $\alpha_{ij}$. A gated residual branch writes
\begin{equation}
\widehat{\mathbf z}_i=\mathbf z_i+
G_\theta\!\left(\mathbf z_i,
\sum_{j\in C(i)}\alpha_{ij}W_V\mathbf z_j\right).
\label{eq:refinement}
\end{equation}
One update per anchor preserves the $B\!\to\!B$ visual interface and its
source metadata. Only the refiner is trained, using answer supervision and
KL distillation from Full Image. Zero-initialized residuals start from the
unrefined pathway; disabling the branch recovers training-free selection.
Context access incurs additional computation despite the unchanged output
length. Supplementary Sec.~\ref{app:refinement} specifies the architecture and loss.

\section{Experiments}
\label{sec:experiments}

We evaluate the compact-input finding, then separate effects of candidate
source, retention rule, and optional context.

\paragraph{Experimental setup.}
\label{sec:experimental-setup}
We evaluate Qwen2-VL-7B, Qwen3-VL-8B, and Qwen3-VL-32B
\citep{wang2024qwen2vl,bai2025qwen3vl} on 13 benchmarks.
Table~\ref{tab:part-i} reports six benchmarks
\citep{hong2025motionbench,fu2025videomme,fu2026videommev2,
xiao2021nextqa,patraucean2023perception,li2026perceptioncomp};
Supplementary Table~\ref{tab:general-exact} reports the remaining seven
\citep{mangalam2023egoschema,zhang2024vinoground,zhang2025videoads,
wang2025lvbench,li2024mvbench,wu2021star,li2023intentqa}.
Accuracy is in percent; Avg.(6) and Avg.(13) are unweighted benchmark means.
T-free denotes inference without additional compressor training.

Standard runs use 32 frames and target 1.5K visual tokens; Native and
Image are uncapped references. Baselines follow official implementations
adapted to the backbone and budget
\citep{fan2026flashvid,liu2025vidcom2,shen2025fastvid,yang2025visionzip},
with FastV, SparseVLM, and DyCoke additionally tested on Qwen3-VL-8B
\citep{chen2024fastv,zhang2025sparsevlm,tao2025dycoke}.
Complete-pathway scores average three seeds under shared inference settings.
Evaluation settings, hardware, and evaluation groups are summarized in
Supplementary Sec.~\ref{app:extended_results}.

\subsection{Aggregate Accuracy at a Compact Budget}
\label{sec:main-results}

\begin{figuretext}{236.24bp}
The training-free pathway maintains comparable Full Image aggregate
accuracy across all three backbones with approximately 1.5K anchors
(Fig.~\ref{fig:ablation-summary}(a)), using
27.6\%, 34.7\%, and 34.7\% of its visual tokens.

Table~\ref{tab:tf-main} reports TF and complete-pathway evaluations;
matched context-access effects are tested separately below.
The 8B TF task-level results ($\dagger$) have Avg.(6)=57.01
and Avg.(13)=62.64.
Table~\ref{tab:part-i} gives the task scores. The complete pathway exceeds
the strongest budget-matched baseline in Avg.(6) by 2.96, 1.92, and
1.82 pp on 7B, 8B, and 32B, respectively. These gains are not uniform:
32B PerceptionComp trails FastVID, 31.75 versus 34.48.
\besidefigure
\centering
\includegraphics[width=\linewidth]{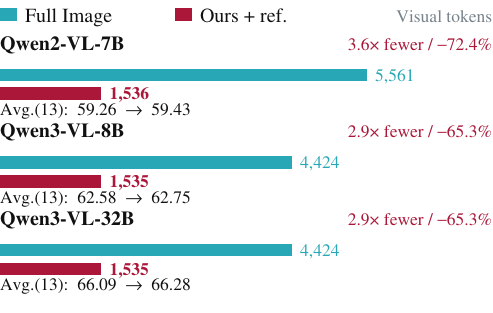}
\captionof{figure}{\textbf{Image-level accuracy with fewer visual tokens.} Bars compare Full Image with Ours including refinement. Counts and reductions describe visual tokens; Avg.(13) is mean accuracy across 13 benchmarks.}
\label{fig:part-ii}

\end{figuretext}

\subsection{Image Candidates and Joint Selection}
\label{sec:pool_control}

\begin{figuretext}{237.58bp}
Fig.~\ref{fig:core-controls}(a) separates the candidate pathway and selector
without refinement. Switching FastVID to Image adds 1.66 pp. Within a
common pool, Ours exceeds FastVID by 2.78 pp on Image and 1.55 pp on
native candidates; the point-estimate interaction is 1.23 pp.
Thus the finding involves both source and retention rule. The source
factor also changes candidate counts and interfaces
(Supplementary Sec.~\ref{app:pool_control}).

Table~\ref{tab:matched-near} adds CDPruner and V-CAST
\citep{zhang2025cdpruner,lin2026vcast} under the same 32-frame Image pool,
1,536-token budget, and no refiner. Joint improves over V-CAST by
3.10, 2.81, 3.06, and 3.25 pp on Video-MME, MVBench, LVBench, and
P.Test, respectively. Seed scores and paired differences are in
Supplementary Sec.~\ref{app:matched_near}.
\besidefigure
\centering
\includegraphics[width=\linewidth]{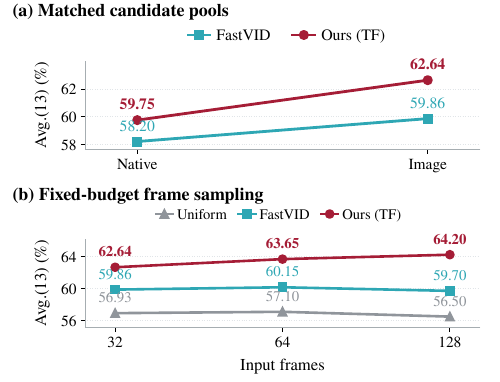}
\captionof{figure}{\textbf{Matched source control and denser sampling.} Qwen3-VL-8B retains 1,536 tokens without refinement. (a) varies candidate pools and selectors; (b) varies input frames; scores are mean accuracy over 13 benchmarks.}
\label{fig:core-controls}
\end{figuretext}

\begin{table}[tbp]
\begin{minipage}[t]{0.48\linewidth}
\vspace{0pt}
\captionsetup{hypcap=false}
\captionof{table}{\textbf{Matched Image-pool results.} Qwen3-VL-8B uses 32 frames and 1,536 tokens, without refinement. Scores (\%) are three-seed means $\pm$ SD; MME/PT denote Video-MME/Perception Test.}
\label{tab:matched-near}
\centering
\fontsize{8.5}{10.3}\selectfont
\setlength{\tabcolsep}{1.2pt}
\renewcommand{\arraystretch}{1.15}
\begin{tabularx}{\linewidth}{@{}l*{4}{>{\centering\arraybackslash}X}@{}}
\toprule
Method & MME & MVBench & LVBench & PT \\
\midrule
CDPruner & 61.50{\scriptsize$\pm$0.10} & 63.02{\scriptsize$\pm$0.08} & 38.55{\scriptsize$\pm$0.06} & 64.42{\scriptsize$\pm$0.08} \\
V-CAST & 62.30{\scriptsize$\pm$0.10} & 64.10{\scriptsize$\pm$0.13} & 39.20{\scriptsize$\pm$0.10} & 65.40{\scriptsize$\pm$0.13} \\
Unary & 64.15{\scriptsize$\pm$0.09} & 65.80{\scriptsize$\pm$0.07} & 41.50{\scriptsize$\pm$0.10} & 68.10{\scriptsize$\pm$0.08} \\
\midrule
\rowcolor{IFTable}
\textbf{Joint (Ours)} & \textbf{65.40{\scriptsize$\pm$0.09}} & \textbf{66.91{\scriptsize$\pm$0.10}} & \textbf{42.26{\scriptsize$\pm$0.12}} & \textbf{68.65{\scriptsize$\pm$0.07}} \\
\bottomrule
\end{tabularx}
\end{minipage}\hfill
\begin{minipage}[t]{0.48\linewidth}
\vspace{0pt}
\captionsetup{hypcap=false}
\captionof{table}{\textbf{Same-H200 cost.}\newline Qwen3-VL-8B uses 32 frames, BF16, and batch size one. TTFT/E2E include selection (ms); GPU is peak memory (GB), and ratios use EOS-terminated Full Image.}
\label{tab:matched-h200}
\centering
\fontsize{8.5}{10.3}\selectfont
\setlength{\tabcolsep}{2pt}
\renewcommand{\arraystretch}{1.15}
\begin{tabularx}{\linewidth}{@{}l*{4}{>{\centering\arraybackslash}X}@{}}
\toprule
Method & Tokens & TTFT & E2E & GPU \\
\midrule
Conv3D & $\sim$2,212 & 145.2 & 380.5 & 32.4 \\
Image & 4,424 & 212.8 & 516.4 & 37.8 \\
\rowcolor{IFTable}
\textbf{Ours TF} & \textbf{1,536} & \textbf{128.2} & \textbf{314.9} & \textbf{31.5} \\
\midrule
vs. Image & $-65.3\%$ & $1.66\times$ & $1.64\times$ & $-16.7\%$ \\
\bottomrule
\end{tabularx}
\end{minipage}

\end{table}

\subsection{Task Coverage and Input Sampling}
\label{sec:cross-benchmark-generalization}

Fig.~\ref{fig:part-ii} summarizes token compression and macro-accuracy.
Against the strongest budget-matched score per task, our complete pathway improves EGO,
VideoAds, LVBench, and STAR across all three backbones. MVBench remains below the strongest compressed
baseline by 4.69, 2.13, and 3.01 pp; all scores are in Supplementary
Table~\ref{tab:general-exact}.

At 1.5K output tokens, increasing sampling from 32 to 128 frames raises
our score from 62.64 to 64.20, while FastVID-on-Image changes from 59.86
to 59.70 (Fig.~\ref{fig:core-controls}(b)). The gap grows from 2.78 to
4.50 pp. All three frame counts use the directed scene gate and bounded floors;
at 128 frames, a floor of six leaves 768 dynamic slots
(Supplementary Sec.~\ref{app:floor_scaling}). Encoding cost grows with sampling
density. This tests denser sampling, not a new duration distribution.

\subsection{Joint Retention and Temporal Context}
\label{sec:ablation}

\begin{figuretext}{232.64bp}
\paragraph{Joint retention versus independent endpoint rewards.}
Fig.~\ref{fig:task-component-gains} holds the Image pool, frame quotas,
DPP kernels, graph, and greedy move rules fixed. Independent DPP obtains
60.53 Avg.(13), additive endpoint rewards reach 61.20, and joint retention
reaches 62.64. Joint retention adds 1.44 pp over additive rewards on Avg.(13).
The three-seed comparison gives paired Joint$-$additive differences of
$1.25\pm0.02$, $1.11\pm0.03$, $0.76\pm0.03$, and $0.55\pm0.01$ pp
on Video-MME, MVBench, LVBench, and P.Test (sample SD).
The additive reward preserves the kernel by acting outside logdet
(Supplementary Sec.~\ref{app:reward_control}); Avg.(13) remains mean-only.
Allocation/coordination ablations are in Supplementary
Table~\ref{tab:component_objective_ablation}.

Both reward variants use the same temporal graph. Their difference is
whether a correspondence rewards an endpoint or a retained pair.
\besidefigure
\centering
\includegraphics[width=\linewidth]{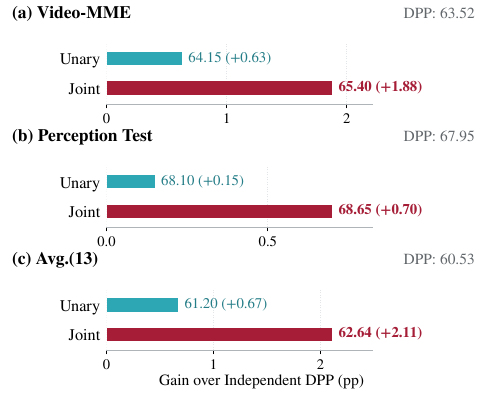}
\captionof{figure}{\textbf{Fixed-kernel rewards (8B).} Pool, quotas, kernels, graph, and solver moves are fixed, without refinement. Bars show gains over Independent DPP: Unary rewards endpoints, while Joint rewards retained pairs.}
\label{fig:task-component-gains}
\end{figuretext}

The
observed gap supports temporal coordination of membership in the compact
input while preserving selected feature values. It does not identify which
visual facts account for the answer changes.

\begin{figuretext}{232.64bp}
The correspondence intervention tests whether
the cross-frame structure matters: on Qwen2-VL-7B, correct edges beat
absent and shuffled edges by 0.93 and 1.90 pp
(Fig.~\ref{fig:ablation-summary}(b)). The 8B parameter-range bounds limit
losses from default to 0.19 pp (cosine threshold) and 0.14 pp (pair weight)
(Supplementary Table~\ref{tab:sensitivity_ranges}).
These are tested-range minima; benchmark-specific presets and
parameter selection are discussed in
Supplementary Sec.~\ref{app:parameter_scope}.

The separate context-access comparison (Fig.~\ref{fig:ablation-summary}(c--d),
Supplementary Table~\ref{tab:refinement_primary}) holds the 65K training
examples, two epochs, initialization, and optimization recipe fixed;
only the context-access mask changes. Extra context improves Video-MME over this study's TF baseline by
2.22 pp; LVBench and P.Test change by $-0.02$ and $-0.01$ pp.
\besidefigure
\centering
\includegraphics[width=\linewidth]{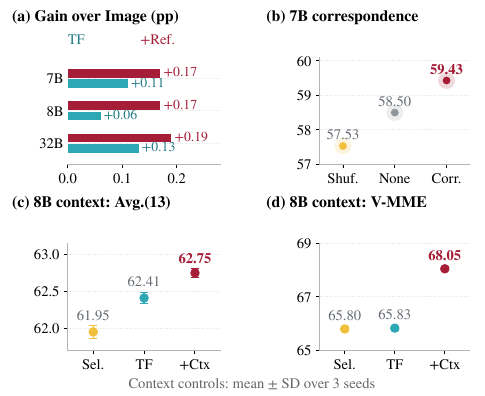}
\captionof{figure}{\textbf{Controls.} (a--b) test aggregate accuracy and correspondence structure. The separate 8B context study (c--d) compares no refiner (TF), selected-only neighbors (Sel.), and unselected context (+Ctx); bars denote three-seed SDs.}
\label{fig:ablation-summary}
\end{figuretext}

Avg.(13) reaches 62.75 versus 61.95 for selected-only refinement and
62.41 for TF. The 0.34-pp gain keeps refinement optional.

In these panels, Sel.\ uses selected neighbors; +Ctx also accesses
unselected candidates.
Error bars are marginal seed SDs, not paired uncertainties.
Proposition~\ref{prop:cut} identifies omitted interactions without
predicting accuracy gains. The matched configuration is given in
Supplementary Sec.~\ref{app:refinement_followup}.

\subsection{Inference Cost}
\label{sec:efficiency}

\noindent\begin{minipage}{\linewidth}
\begin{figuretext}{237.58bp}
Table~\ref{tab:matched-h200} measures Qwen3-VL-8B on one H200 with BF16,
BS=1, and the same 32-frame inputs. The timer starts from sampled video
pixels and includes visual encoding, allocation/selection, and LLM prefill
before the first generated token. Selection takes 14.5 ms. Relative to
Full Image, TF reduces TTFT from 212.8 to 128.2 ms ($1.66\times$), E2E
from 516.4 to 314.9 ms ($1.64\times$), and peak memory from 37.8 to
31.5 GB (16.7\%). Generation permits EOS with \texttt{max\_new\_tokens}=128;
the E2E ratio describes this inference workload, not fixed-length decoding.
\besidefigure
\centering
\includegraphics[width=\linewidth]{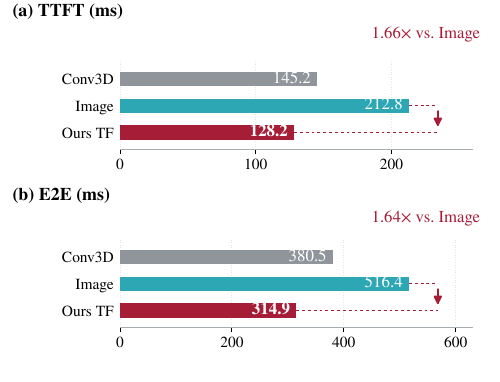}
\end{figuretext}
\captionsetup{hypcap=false}
\captionof{figure}{\textbf{Same-H200 latency, Qwen3-VL-8B.} Methods share 32 frames, BF16, and batch size one. First-token (a) and end-to-end (b) latencies include selection; speedups compare Full Image under EOS-terminated generation.}
\label{fig:source-cost}
\end{minipage}\par

\begin{figuretext}{236.24bp}
Fig.~\ref{fig:multibudget-main} extends accuracy--budget evaluation over
768--3,072 tokens and three backbones. Sweep latencies and the
32B efficiency profile are separate from matched-H200 timing
(Supplementary Sec.~\ref{app:multibudget-pareto}).

\subsection{Source Evidence and Temporal Limits}
\label{sec:qualitative_main}

Fig.~\ref{fig:interaction_order} illustrates operator ordering;
Supplementary Fig.~\ref{fig:retention-toy} computes the three objectives
on synthetic features, complementing the real-task control.

Fig.~\ref{fig:source-detail} uses generated scenes and manual boxes to
illustrate source access. It depicts source crops rather than encoder
reconstructions or measured retention masks.
Fig.~\ref{fig:qualitative-main} adds four real-video comparisons on the
8B backbone. Ours answers the head-motion, shoe-lace interaction, and
object-counting cases correctly.
\besidefigure
\centering
\includegraphics[width=\linewidth]{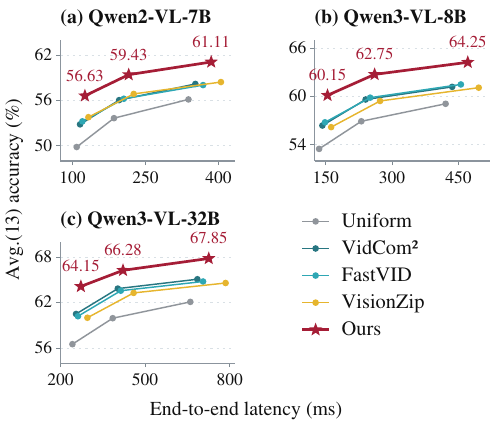}
\captionof{figure}{\textbf{Multi-budget accuracy--latency sweep across three backbones.} Curves connect 768-, 1.5K-, and 3K-token settings. Higher Avg.(13) and lower latency are preferable; these timings are separate from the H200 control.}
\label{fig:multibudget-main}

\end{figuretext}

\begin{figure*}[t]
\centering
\includegraphics[width=\linewidth]{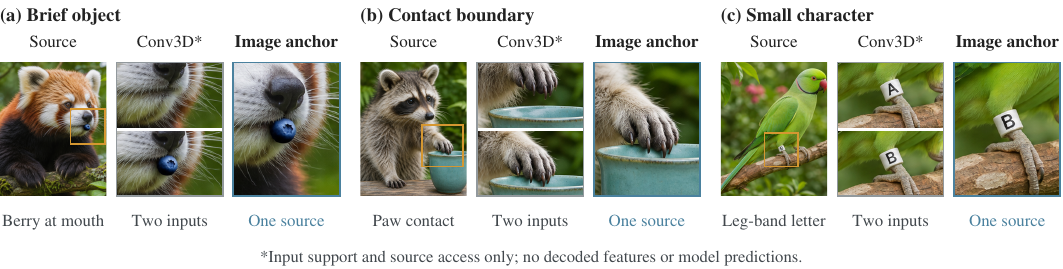}
\caption{\textbf{Joint input support and separate source anchors (schematic).} A brief object, contact boundary, and small character contrast mixed input support with image anchors. Crops illustrate source access, not decoded features or retention masks.}
\label{fig:source-detail}
\par\vspace{10pt}
\centering
\includegraphics[width=\linewidth]{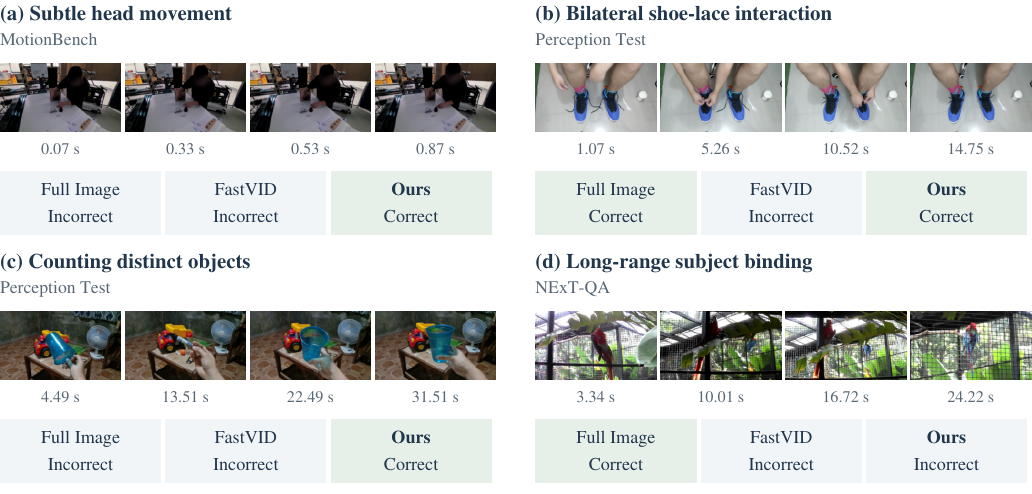}
\caption{\textbf{Real-video successes and a temporal binding failure.} On Qwen3-VL-8B, Ours succeeds on local motion (a), interaction (b), and counting (c), but fails on long-range binding (d). Full Image succeeds in (b,d); FastVID fails in all four. Green marks correct predictions; timestamped frames show input context, not selected-token masks.}
\label{fig:qualitative-main}
\end{figure*}

Full Image also succeeds on the shoe-lace
case, but misses the other two; FastVID misses all three. The parrot case
provides a counterexample: Full Image identifies the subject of the late
action, while both compact pathways shown choose an incorrect answer.
The timestamps place the cues within the input clips; the displayed
frames are not retention masks. These selected cases illustrate distinct
temporal demands, not their frequency in the evaluation set.
The counting case spans more time than the binding failure, so these
examples cannot attribute failure to duration alone.
Supplementary Sec.~\ref{app:qualitative} provides the full method
comparisons and additional cases.
These examples and the MVBench drops delimit the temporal findings;
retained-token tracing and interventions remain needed to establish which
evidence changes answers.

\section{Related Work}
\label{sec:related_work}

\paragraph{Visual representations for video.}
Video encoders aggregate appearance, use 3D convolutions, or apply
space--time attention
\citep{karpathy2014large,simonyan2014two,
tran2015learning,carreira2017quo,bertasius2021space,arnab2021vivit}.
Video-language interfaces adapt these representations
\citep{maaz2024videochatgpt,zhang2023videollama,cheng2024videollama2,weng2024longvlm};
motion-aware studies examine temporal fusion and compressed-video inputs
\citep{hong2025motionbench,zhao2025ema}.
Framewise image features also support video understanding
\citep{li2025onevision,zhang2025llavavideo}. We study their joint retention
under a compact output budget.

\paragraph{Learned interfaces and long videos.}
Flamingo, BLIP-2, TokenPacker, and LLaVA-Mini learn compact visual interfaces
\citep{alayrac2022flamingo,li2023blip2,li2025tokenpacker,zhang2025llavamini}.
Long-video approaches use key-frame queries, memory, extended context,
or hierarchical and learned compression
\citep{tan2024koala,he2024malmm,zhang2025longva,
li2026videochatflash,zhang2026onetoken}.
SlowFast-LLaVA and Q-Frame adapt coverage and resolution
\citep{xu2024slowfastllava,zhang2025qframe}; our quotas act on encoded frames.

\paragraph{Visual-token budgets.}
EViT and ToMe reduce transformer tokens \citep{liang2022evit,bolya2023tome};
PruMerge, VisionZip, and DivPrune retain informative or diverse representations
\citep{shang2025prumerge,yang2025visionzip,alvar2025divprune}.
Compression before or within the LLM uses pruning, merging, and cache reduction
\citep{chen2024fastv,zhang2025sparsevlm,zhong2025aim,shao2025holitom,tao2025dycoke},
including video-specific spatiotemporal methods
\citep{shen2025fastvid,liu2025vidcom2,fan2026flashvid,du2026unified}.
NVILA studies system efficiency \citep{liu2025nvila}; pruning behavior also
depends on depth and task \citep{wang2026informationhorizon}.
Our focus is source-preserving selection at a fixed budget.

\paragraph{Conditional diversity and temporal retention.}
CDPruner conditions a DPP on instruction relevance
\citep{zhang2025cdpruner}; HieraVid applies DPP selection within a
hierarchical video-compression pipeline \citep{guo2026hieravid}.
V-CAST routes frame budgets using semantic curvature and retains original
token coordinates \citep{lin2026vcast}. Conditional diversity, adaptive
budgets, and coordinate preservation are established building blocks.
Our study separates candidate-source and retention effects. At fixed
quotas, an explicit retained-pair reward couples cross-frame membership
without merging selected features; the same indices gather positions and
DeepStack streams. The fixed-kernel unary control tests this mechanism.
Supplementary Sec.~\ref{app:related_work} compares the
published designs and their contribution boundaries in detail.

Sequential DPPs model dependent video-summary subsets
\citep{gong2014seqdpp}. Our objective coordinates question-conditioned
patch tokens under fixed frame quotas, with a source-aligned VLM
interface. Temporal dependence alone is established; the retained-pair
reward specifies our additional mechanism.

\section{Conclusion}
\label{sec:conclusion}

Our experiments identify an effective encoding order for the evaluated
Qwen systems: represent frames independently, then retain source tokens
jointly. Across 13 benchmarks, this training-free pathway maintains
comparable Full Image aggregate accuracy with 28\%--35\% of its visual tokens.
Candidate-pool and fixed-kernel controls show contributions from both the
image pathway and joint retention. The same selection indices preserve
source embeddings, positions, and intermediate streams, separating
coordinated membership from feature mixing. Under EOS-terminated generation,
single-H200 Qwen3-VL-8B inference is $1.64\times$ faster in the measured workload.
Optional refinement adds context at fixed output length, with small
aggregate gains in the evaluated settings. Motion-task failures and
long-range binding errors delimit these findings. Broader backbone tests
and direct interventions remain needed to assess generality and establish
which retained evidence changes answers. These results motivate evaluating
the candidate representation and retention rule together. At a strict
visual-token budget, preserving source evidence and coordinating its
survival can matter more than introducing temporal mixing at the earliest
stage. Independent frame encoding
therefore need not imply independent frame selection.

\FloatBarrier
\paragraph*{LLM Usage Statement.}
We use GPT-6 Astra and Codex as supporting tools during manuscript
preparation and software development. Their role is to help improve
wording, grammar, and the clarity of the exposition, and to assist with
implementing and debugging selected code components. All text and code
produced or revised with these tools undergo careful review by the
authors, who make any necessary corrections and determine the final
content. Responsibility for the manuscript's integrity, the
implementation's reliability, and the validity of the experimental
results and conclusions rests entirely with the authors.

\FloatBarrier
\par\addvspace{\baselineskip}
\appendix
\allowdisplaybreaks[1]
\definecolor{IFTable}{RGB}{235,225,227}
\definecolor{IFReference}{RGB}{250,250,250}
\definecolor{IFProtocol}{RGB}{247,247,247}
\definecolor{IFContext}{RGB}{245,245,245}
\noindent{\large\bfseries Supplementary Material Overview}\par\nobreak
This supplement provides evaluation controls, additional results,
qualitative cases, and implementation details for Image-First Video
Encoding. Section letters and page references follow the supplement.

\noindent\textbf{Contents}\par\nobreak
\begingroup
\setlength{\tabcolsep}{0pt}
\renewcommand{\arraystretch}{1.05}
\begin{tabular*}{\linewidth}{@{}l@{\hspace{.5em}}l@{\extracolsep{\fill}}r@{\hspace{1.2em}}l@{\hspace{.5em}}lr@{}}
A & Evaluation settings and controls & \supppages{app:extended_results}{end:extended_results,tab:refinement_primary,tab:matched-h200-full} & H & Prediction-error decomposition & \supppages{app:information}{end:information}\\
B & Computed co-selection example & \supppages{app:detail-probes}{end:detail-probes,tab:retention-toy-exact} & I & Selection--interaction identity & \supppages{app:cut}{end:cut}\\
C & Extended related work & \supppages{app:related_work}{end:related_work,tab:closest_designs} & J & Temporal refinement & \supppages{app:refinement}{end:refinement}\\
D & Additional experimental results & \supppages{app:additional-results}{end:additional-results,tab:general-exact} & K & Selector and hyperparameters & \supppages{app:selector_spec}{end:selector_spec,tab:selector_hparams}\\
E & Ablations and training details & \supppages{app:ablations}{end:ablations,fig:latency-breakdown,tab:eff} &  & References & \supppages{supp:references}{end:bib}\\
F & Qualitative analysis & \supppages{app:qualitative}{end:qualitative,fig:qualitative-success-2,tab:main_case_predictions} & V1 & Selection and sampling plots & \pageref{fig:supp-selection-visual}\\
G & Coordinated selection analysis & \supppages{app:selection}{end:selection} & V2 & Same-H200 efficiency plots & \pageref{fig:supp-h200-visual}\\
\end{tabular*}
\endgroup
\par

\section{Evaluation Settings and Additional Controls}
\label{app:extended_results}

The controls test the central finding: independent frame encoding followed
by joint retention can maintain aggregate accuracy under a compact budget.
They separate candidate-source, selection, and optional context effects.
We evaluate Qwen2-VL-7B, Qwen3-VL-8B, and Qwen3-VL-32B on the 13
benchmarks listed in the main paper. Accuracies are percentages;
differences are percentage points (pp). Avg.(6) and Avg.(13) are
unweighted benchmark means. Visual-token counts exclude question text
and generated answers.

\paragraph{Evaluation groups.}
The TF task-level evaluation, fixed-kernel reward control, and matched-H200
profile use Qwen3-VL-8B with the refiner disabled. The context-access
control and complete-pathway evaluation form separate
groups (Table~\ref{tab:evaluation_scope}). Task gains and uncertainty are
compared within groups; an identical rounded macro-average does not make
two groups a paired experiment.

\begin{table*}[!t]
\caption{\textbf{Scope of the reported evaluations.} Rows identify backbones, controlled factors, and corresponding results. Separate groups must not be treated as paired experiments merely because their rounded scores match.}
\label{tab:evaluation_scope}
\centering
\begingroup
\small
\setlength{\tabcolsep}{4pt}
\renewcommand{\arraystretch}{1.16}
\begin{tabularx}{\linewidth}{@{}llX@{}}
\toprule
\textbf{Group} & \textbf{Backbone} & \textbf{Scope and corresponding results} \\
\midrule
Complete-pathway evaluation & \shortstack[l]{Three Qwen\\backbones} & Complete-pathway task scores and aggregate TF scores (main Tables~\ref{tab:part-i} and~\ref{tab:tf-main}; Table~\ref{tab:general-exact}). \\
TF task-level evaluation & Qwen3-VL-8B & Complete 13-task TF scores marked $\dagger$; Avg.(6)=57.01, Avg.(13)=62.64. \\
Candidate-pool control & Qwen3-VL-8B & Native/Image candidates crossed with FastVID/Ours; no refiner (main Fig.~\ref{fig:core-controls}(a)). \\
Fixed-kernel reward & Qwen3-VL-8B & Independent, additive, and joint retention under shared quotas, kernels, graph, and solver moves (Table~\ref{tab:reward_control}). \\
Matched nearest methods & Qwen3-VL-8B & CDPruner, V-CAST, additive, and joint on four tasks; three-seed scores and paired differences (Tables~\ref{tab:matched_near_seeds}--\ref{tab:matched_near_paired}). \\
Sampling density & Qwen3-VL-8B & 32/64/128 frames under the directed scene gate and bounded floor reserve (Table~\ref{tab:scaling_accuracy}). \\
Context-access control & Qwen3-VL-8B & Separate matched-training TF/selected-only/unselected-context comparison (Table~\ref{tab:refinement_primary}). \\
Same-H200 timing & Qwen3-VL-8B & Native, Image, and TF timing on one H200, BF16, BS=1 (main Table~\ref{tab:matched-h200}). \\
Component and system studies & 8B / 7B / 32B & Allocation/coordination ablation; 7B system ablations; 32B efficiency profile (Tables~\ref{tab:component_objective_ablation}, \ref{tab:hyper_ablation}, and~\ref{tab:eff}). \\
\bottomrule
\end{tabularx}
\endgroup
\end{table*}

\begin{figure*}[htbp]
\centering
\phantomsection\label{fig:supp-selection-visual}%
\includegraphics[width=\linewidth]{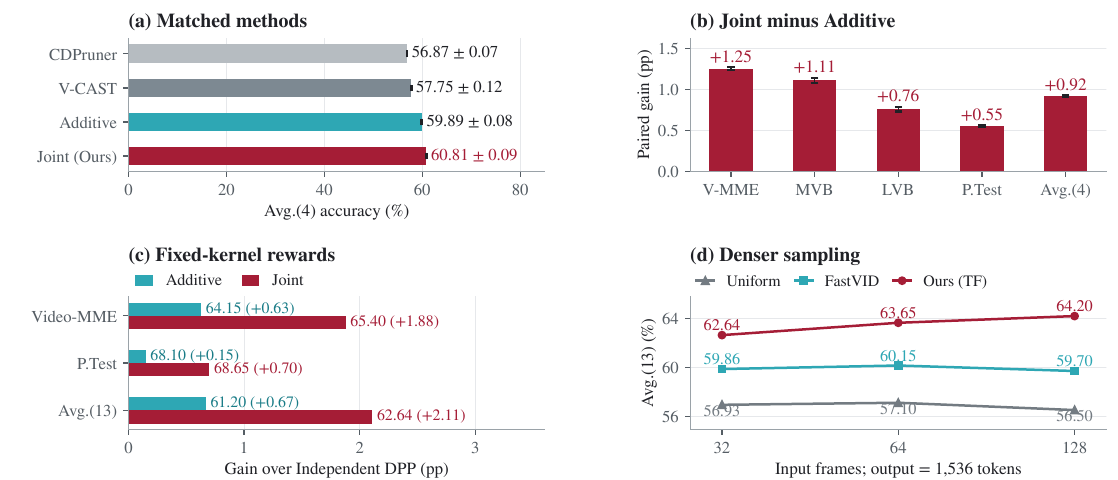}
\caption*{\textbf{V1.} Selection controls and denser sampling. (a) Seed SD; (b) paired-difference SD; (c--d) means (Tables~\ref{tab:matched_near_seeds}--\ref{tab:scaling_accuracy}).}
\end{figure*}

\paragraph{Shared settings and TF task-level evaluation.}
Standard runs use 32 frames. Native video and Full Image are uncapped
references; the 8B TF task-level evaluation passes 1,536 visual tokens.
Within each evaluation group and backbone, samples, prompts, decoding, and answer
parsing are shared; scores average seeds 42, 3407, and 114514.
The 8B TF task means sum to 814.32, giving Avg.(13)=62.64.
The main six sum to 342.06, giving Avg.(6)=57.01.
The $\dagger$ rows report this TF task-level evaluation. They share the
Video-MME/P.Test/Avg.(13) values of the fixed-kernel reward control, while the
65.83/44.05/69.25 TF baseline below belongs to the separate context-access control.
For 7B and 32B, TF remains reported at the aggregate level.

\subsection{Independent Image Candidates and Selection}
\label{app:pool_control}

The matched candidate-pool control uses Qwen3-VL-8B, 32 frames,
approximately 1.5K output tokens, and no refiner. FastVID scores 58.20
on native candidates and 59.86 on Image candidates; Ours scores 59.75
and 62.64, respectively. The selector advantage is 1.55 pp on native
candidates and 2.78 pp on Image candidates. Their difference, 1.23 pp,
is the point-estimate pool-by-selector interaction. The 1.66-pp FastVID
gain from Image candidates and the within-pool selector gains support
distinct contributions. The source factor also changes candidate count
and interface; it does not isolate the convolution kernel alone.

\subsection{Matched CDPruner, V-CAST, and Seed Replication}
\label{app:matched_near}

This comparison uses Qwen3-VL-8B, 32 frames, the independent Image
candidate pool, exactly 1,536 output visual tokens, and no refiner.
Table~\ref{tab:matched_near_seeds} reports all three seed scores
for CDPruner, V-CAST, the additive control, and joint selection.
P.Test in main Table~\ref{tab:matched-near} denotes Perception Test.
LVBench is not LongVideoBench.
We evaluate the external methods as Image-pool adaptations under the same
backbone and output budget. The allocator, kernel, and graph are held
fixed specifically for the additive--joint mechanism control, not imposed
on the external methods.

\begin{table*}[!htbp]
\begin{minipage}[t]{0.487\linewidth}
\vspace{0pt}
\caption{\textbf{Matched four-task results and three-seed scores.} Qwen3-VL-8B uses 32 frames, 1,536 Image tokens, and no refiner. Accuracy (\%) is mean $\pm$ sample SD; Avg.(4) is the unweighted task mean.}
\label{tab:matched_near_seeds}
\centering
\fontsize{9}{11}\selectfont
\setlength{\tabcolsep}{0pt}
\renewcommand{\arraystretch}{1.04}
\begin{tabularx}{\linewidth}{@{}>{\raggedright\arraybackslash}p{0.22\linewidth}*{3}{>{\hsize=.8\hsize\linewidth=\hsize\raggedleft\arraybackslash}X}>{\hsize=1.6\hsize\linewidth=\hsize\raggedleft\arraybackslash}X@{}}
\toprule
Method & 42 & 3407 & 114514 & Mean $\pm$ SD \\
\midrule
\multicolumn{5}{@{}l}{\textit{Video-MME}} \\
CDPruner & 61.45 & 61.62 & 61.43 & $61.50\pm0.10$ \\
V-CAST & 62.25 & 62.42 & 62.23 & $62.30\pm0.10$ \\
Additive & 64.08 & 64.25 & 64.12 & $64.15\pm0.09$ \\
\rowcolor{IFTable}
Joint (Ours) & 65.35 & 65.50 & 65.35 & $65.40\pm0.09$ \\
\midrule
\multicolumn{5}{@{}l}{\textit{MVBench}} \\
CDPruner & 62.94 & 63.10 & 63.01 & $63.02\pm0.08$ \\
V-CAST & 64.05 & 64.25 & 64.00 & $64.10\pm0.13$ \\
Additive & 65.75 & 65.88 & 65.77 & $65.80\pm0.07$ \\
\rowcolor{IFTable}
Joint (Ours) & 66.85 & 67.02 & 66.86 & $66.91\pm0.10$ \\
\midrule
\multicolumn{5}{@{}l}{\textit{LVBench}} \\
CDPruner & 38.55 & 38.61 & 38.49 & $38.55\pm0.06$ \\
V-CAST & 39.15 & 39.32 & 39.13 & $39.20\pm0.10$ \\
Additive & 41.42 & 41.61 & 41.47 & $41.50\pm0.10$ \\
\rowcolor{IFTable}
Joint (Ours) & 42.15 & 42.38 & 42.25 & $42.26\pm0.12$ \\
\midrule
\multicolumn{5}{@{}l}{\textit{P.Test}} \\
CDPruner & 64.34 & 64.49 & 64.42 & $64.42\pm0.08$ \\
V-CAST & 65.35 & 65.55 & 65.30 & $65.40\pm0.13$ \\
Additive & 68.02 & 68.18 & 68.10 & $68.10\pm0.08$ \\
\rowcolor{IFTable}
Joint (Ours) & 68.58 & 68.72 & 68.65 & $68.65\pm0.07$ \\
\midrule
\multicolumn{5}{@{}l}{\textit{Avg.(4)}} \\
CDPruner & 56.82 & 56.96 & 56.84 & $56.87\pm0.07$ \\
V-CAST & 57.70 & 57.89 & 57.67 & $57.75\pm0.12$ \\
Additive & 59.82 & 59.98 & 59.87 & $59.89\pm0.08$ \\
\rowcolor{IFTable}
Joint (Ours) & 60.73 & 60.91 & 60.78 & $60.81\pm0.09$ \\
\bottomrule
\end{tabularx}
\end{minipage}
\hfill
\begin{minipage}[t]{0.487\linewidth}
\vspace{0pt}
\caption{\textbf{Joint minus each baseline, paired by seed.} Differences are in percentage points; positive values favor Joint. Mean $\pm$ SD summarizes three paired seed differences, not a video-level confidence interval.}
\label{tab:matched_near_paired}
\centering
\fontsize{9}{11}\selectfont
\setlength{\tabcolsep}{0pt}
\renewcommand{\arraystretch}{1.04}
\begin{tabularx}{\linewidth}{@{}>{\raggedright\arraybackslash}p{0.22\linewidth}*{3}{>{\hsize=.8\hsize\linewidth=\hsize\raggedleft\arraybackslash}X}>{\hsize=1.6\hsize\linewidth=\hsize\raggedleft\arraybackslash}X@{}}
\toprule
Baseline & 42 & 3407 & 114514 & Mean $\pm$ SD \\
\midrule
\multicolumn{5}{@{}l}{\textit{Video-MME}} \\
\rule[-4.576pt]{0pt}{15.253pt}CDPruner & 3.90 & 3.88 & 3.92 & \cellcolor{IFPairTint}$3.90\pm0.02$ \\
\rule[-4.576pt]{0pt}{15.253pt}V-CAST & 3.10 & 3.08 & 3.12 & \cellcolor{IFPairTint}$3.10\pm0.02$ \\
\rule[-4.576pt]{0pt}{15.253pt}Additive & 1.27 & 1.25 & 1.23 & \cellcolor{IFPairTint}$1.25\pm0.02$ \\
\midrule
\multicolumn{5}{@{}l}{\textit{MVBench}} \\
\rule[-4.576pt]{0pt}{15.253pt}CDPruner & 3.91 & 3.92 & 3.85 & \cellcolor{IFPairTint}$3.89\pm0.04$ \\
\rule[-4.576pt]{0pt}{15.253pt}V-CAST & 2.80 & 2.77 & 2.86 & \cellcolor{IFPairTint}$2.81\pm0.05$ \\
\rule[-4.576pt]{0pt}{15.253pt}Additive & 1.10 & 1.14 & 1.09 & \cellcolor{IFPairTint}$1.11\pm0.03$ \\
\midrule
\multicolumn{5}{@{}l}{\textit{LVBench}} \\
\rule[-4.576pt]{0pt}{15.253pt}CDPruner & 3.60 & 3.77 & 3.76 & \cellcolor{IFPairTint}$3.71\pm0.10$ \\
\rule[-4.576pt]{0pt}{15.253pt}V-CAST & 3.00 & 3.06 & 3.12 & \cellcolor{IFPairTint}$3.06\pm0.06$ \\
\rule[-4.576pt]{0pt}{15.253pt}Additive & 0.73 & 0.77 & 0.78 & \cellcolor{IFPairTint}$0.76\pm0.03$ \\
\midrule
\multicolumn{5}{@{}l}{\textit{P.Test}} \\
\rule[-4.576pt]{0pt}{15.253pt}CDPruner & 4.24 & 4.23 & 4.23 & \cellcolor{IFPairTint}$4.23\pm0.01$ \\
\rule[-4.576pt]{0pt}{15.253pt}V-CAST & 3.23 & 3.17 & 3.35 & \cellcolor{IFPairTint}$3.25\pm0.09$ \\
\rule[-4.576pt]{0pt}{15.253pt}Additive & 0.56 & 0.54 & 0.55 & \cellcolor{IFPairTint}$0.55\pm0.01$ \\
\midrule
\multicolumn{5}{@{}l}{\textit{Avg.(4)}} \\
\rule[-4.576pt]{0pt}{15.253pt}CDPruner & 3.91 & 3.95 & 3.94 & \cellcolor{IFPairTint}$3.93\pm0.02$ \\
\rule[-4.576pt]{0pt}{15.253pt}V-CAST & 3.03 & 3.02 & 3.11 & \cellcolor{IFPairTint}$3.06\pm0.05$ \\
\rule[-4.576pt]{0pt}{15.253pt}Additive & 0.92 & 0.93 & 0.91 & \cellcolor{IFPairTint}$0.92\pm0.01$ \\
\bottomrule
\end{tabularx}
\end{minipage}
\end{table*}

The Full Image references are 67.96, 66.25, 43.87, and 69.28
on Video-MME, MVBench, LVBench, and P.Test, respectively. The four-task
macro-average is computed for each seed before its mean and SD; it is
not Avg.(13). Rounded per-seed macro-averages in the table are for display;
the final statistics use the unrounded averages.

\paragraph{Paired seed differences.}
For task $d$, baseline $m$, and seed $r$, define
$\Delta_{d,m,r}=A_{d,\mathrm{joint},r}-A_{d,m,r}$.
Table~\ref{tab:matched_near_paired} reports the mean of these three
differences and their sample SD, using denominator $3-1$. This directly
measures across-seed variation in the observed advantage; it is not
computed by combining two marginal SDs.

Joint exceeds V-CAST by 3.10, 2.81, 3.06, and 3.25 pp, and exceeds
CDPruner by 3.90, 3.89, 3.71, and 4.23 pp on the four tasks in order.
The paired Joint--additive gaps are $1.25\pm0.02$, $1.11\pm0.03$,
$0.76\pm0.03$, and $0.55\pm0.01$ pp.
These results support the comparison on the tested four tasks and this
backbone. They do not establish a 13-task or cross-backbone ranking
against CDPruner or V-CAST. The full 13-task reward means in
Table~\ref{tab:reward_control} remain without paired seed uncertainty.
No video-level bootstrap interval is inferred from these aggregate seed
accuracies.

\subsection{Joint Retention under Fixed Kernels}
\label{app:reward_control}

This control holds the Image candidate pool, frame quotas, DPP kernels,
adjacent-frame graph, reward weight $\gamma$, and feasible singleton/pair
moves fixed. All variants are training-free and retain 1,536 source tokens.
Let $D(S)=\sum_t\log\det L_t[S_t,S_t]$. Independent DPP uses $D(S)$;
the additive control uses
\begin{equation}
\begin{aligned}
R_i&=\frac12\sum_{j:\{i,j\}\in E}w_{ij},\\
J_{\mathrm{add}}(S)&=D(S)+\gamma\sum_{i\in S}R_i.
\end{aligned}
\label{eq:real_additive_control}
\end{equation}
The reward is outside logdet: neither token quality nor $L_t$ is changed.
The factor $1/2$ equates the additive and joint graph rewards when both
endpoints of every edge are retained.

\paragraph{What changes at the selection boundary.}
Define $\partial_E S=\{\{i,j\}\in E:|S\cap\{i,j\}|=1\}$. Then
\begin{equation}
 J_{\mathrm{add}}(S)-J(S)
 =\frac{\gamma}{2}\sum_{\{i,j\}\in\partial_E S}w_{ij}.
\label{eq:additive_boundary_identity}
\end{equation}
An edge with neither endpoint retained contributes zero to both rewards;
one retained endpoint contributes $\gamma w_{ij}/2$ only to the additive
reward; two retained endpoints contribute $\gamma w_{ij}$ to both.
Thus fixed node rewards also credit correspondences cut by selection,
whereas the joint term credits retained endpoint pairs. This identity
explains the objective distinction without implying that joint selection
must outperform additive rewards on every dataset.

\begin{table}[!htbp]
\caption{\textbf{Fixed-kernel reward control.} Qwen3-VL-8B retains 1,536 tokens without refinement; pool, quotas, kernels, graph, and moves are fixed. Only the reward changes; scores are accuracy (\%) and differences are percentage points.}
\label{tab:reward_control}
\centering
\begingroup
\small
\setlength{\tabcolsep}{4pt}
\renewcommand{\arraystretch}{1.14}
\begin{tabularx}{\linewidth}{@{}l>{\centering\arraybackslash}p{0.34\linewidth}*{3}{>{\centering\arraybackslash}X}@{}}
\toprule
Selector & Added reward outside logdet & Video-MME & P.Test & Avg.(13) \\
\midrule
Independent DPP & $0$ & 63.52 & 67.95 & 60.53 \\
\rowcolor{IFControl}
Additive reward & $\gamma\sum_{i\in S}R_i$ & 64.15 & 68.10 & 61.20 \\
\rowcolor{IFTable}
\textbf{Joint retention (Ours)} & $\gamma\sum_{\{i,j\}\in E}w_{ij}\mathbf1[i,j\in S]$ & \textbf{65.40} & \textbf{68.65} & \textbf{62.64} \\
\midrule
Joint $-$ additive (pp) & --- & $+1.25$ & $+0.55$ & $+1.44$ \\
\bottomrule
\end{tabularx}
\endgroup

\end{table}

Table~\ref{tab:reward_control} reports the observed task means. Joint
retention improves over additive rewards by 1.25 pp on Video-MME,
0.55 pp on P.Test, and 1.44 pp on Avg.(13). It improves over Independent
DPP by 1.88, 0.70, and 2.11 pp, respectively. These full 13-task summaries are mean differences without paired
macro-average uncertainty. The four-task seed replication in
Sec.~\ref{app:matched_near} supplies paired SDs on Video-MME, MVBench,
LVBench, and P.Test; it does not supply a 13-task interval.
The exact synthetic example in Supplementary Sec.~\ref{app:retention_toy} serves a
different purpose: checking the objectives and solver on declared inputs.

\subsection{Denser Sampling with Bounded Floors}
\label{app:dense_sampling}

Table~\ref{tab:scaling_accuracy} increases sampling from 32 to 64 and
128 frames while fixing the output at 1,536 tokens. Ours rises from 62.64
to 63.65 and 64.20; its margin over FastVID-on-Image grows from 2.78
to 3.50 and 4.50 pp. These evaluations use the directed scene-consistency
statistic in Eq.~\ref{eq:scene_directed_coverage} and the bounded floor
rule in Eq.~\ref{eq:floor_half_budget_reserve}. This tests denser sampling
of the evaluated videos, with additional encoding cost, rather than a
different video-duration distribution or a controlled noise intervention.

\begin{table*}[t]
\centering
\caption{\textbf{Denser-sampling results.} Qwen3-VL-8B retains 1,536 Image tokens without refinement as input frames increase. Scores are Avg.(13) accuracy (\%); the directed scene gate and bounded floor reserve apply throughout.}
\label{tab:scaling_accuracy}
\begingroup
\small
\setlength{\tabcolsep}{0pt}
\renewcommand{\arraystretch}{1.14}
\begin{tabularx}{\linewidth}{@{}>{\raggedright\arraybackslash}p{0.36\linewidth}*{3}{>{\centering\arraybackslash}X}@{}}
\toprule
\textbf{Method} & \textbf{32 frames} & \textbf{64 frames} & \textbf{128 frames} \\
\midrule
Uniform-on-Image & 56.93 & 57.10 & 56.50 \\
FastVID-on-Image & 59.86 & 60.15 & 59.70 \\
\rowcolor{IFTable}
Ours (TF) & 62.64 & 63.65 & 64.20 \\
\bottomrule
\end{tabularx}
\endgroup
\par
\addvspace{1.25\baselineskip}
\caption{\textbf{Parameter-range summaries.} Qwen3-VL-8B uses training-free selection. Lower bounds on Avg.(13) accuracy (\%) summarize tested intervals, not behavior outside those ranges.}
\label{tab:sensitivity_ranges}
\begingroup
\small
\setlength{\tabcolsep}{0pt}
\renewcommand{\arraystretch}{1.14}
\begin{tabularx}{\linewidth}{@{}>{\raggedright\arraybackslash}p{0.38\linewidth}*{2}{>{\centering\arraybackslash}X}@{}}
\toprule
\textbf{Parameter} & \textbf{Tested range} & \textbf{Min. Avg.(13)} \\
\midrule
\rule[-5.643pt]{0pt}{18.81pt}Cosine threshold & [0.60, 0.70] & $\geq 62.45$ \\
\rule[-5.643pt]{0pt}{18.81pt}Pair weight $\gamma$ & [1.5, 2.5] & $\geq 62.50$ \\
\bottomrule
\end{tabularx}
\endgroup

\end{table*}

\paragraph{Allocation scope at higher frame counts.}
\label{app:floor_scaling}
At $B=1536$, the common floor is 12, 12, and six at 32, 64, and 128
frames. With adequate capacities, this reserves 384, 768, and 768 slots,
leaving 1,152, 768, and 768 slots for score-based allocation. The reserve
is applied once overall, within the base portion, rather than repeatedly
for each provisional slot type. This prevents floors from exhausting the
128-frame budget; the arithmetic itself does not guarantee accuracy gains.

\subsection{Sensitivity and External Calibration}
\label{app:sensitivity_ranges}
\label{app:parameter_scope}

Table~\ref{tab:sensitivity_ranges} reports lower bounds over the tested
cosine-threshold and pair-weight ranges. Relative to the 62.64 default,
the largest reported decreases are 0.19 and 0.14 pp. These are evaluated
range summaries, not continuous sensitivity curves or significance tests.

The 13 task presets were calibrated by grid search on an external random
5\% subset of Video-ChatGPT validation data
\citep{maaz2024videochatgpt}, disjoint from the evaluated benchmarks.
Benchmark test answers were not used to search or refine parameters.
Validation configurations were mapped to benchmark presets using rules
based on duration and motion characteristics, then fixed before test
inference. This is external calibration with a rule-based transfer step; the two
reported parameter ranges do not establish a single-preset guarantee
across tasks or robustness to arbitrary distribution shifts.

\subsection{Matched Single-H200 Timing}
\label{app:matched_timing}

\begin{table*}[t]
\centering
\phantomsection\label{fig:supp-h200-visual}%
\includegraphics[width=\linewidth]{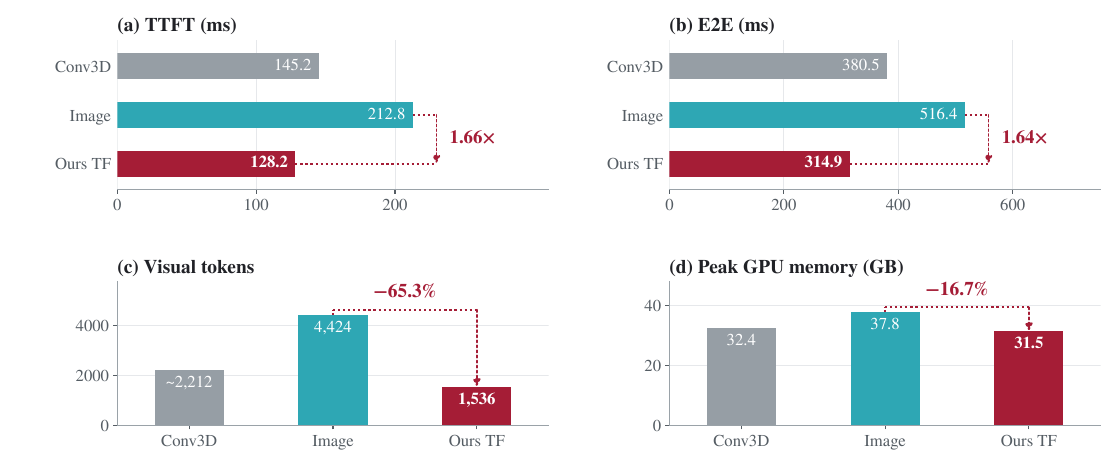}
{\captionsetup{position=bottom,skip=10pt}
\caption*{\textbf{V2.} Matched single-H200 efficiency (Table~\ref{tab:matched-h200-full}). Dotted guides compare Full Image with Ours TF. BF16, batch size one; E2E permits EOS. Selection is included in both timing totals.}}
\par\medskip
\caption{\textbf{Full same-H200 cost profile.} Qwen3-VL-8B uses 32 frames, BF16, and batch size one, with EOS and a 128-token cap. TTFT/E2E include selection; ratios compare Full Image.}
\label{tab:matched-h200-full}
\centering
\begingroup
\small
\setlength{\tabcolsep}{4pt}
\renewcommand{\arraystretch}{1.12}
\begin{tabularx}{\linewidth}{@{}l*{5}{>{\centering\arraybackslash}X}@{}}
\toprule
Method & Visual tokens & Selection (ms) & TTFT (ms) & E2E (ms) & Peak GPU (GB) \\
\midrule
Native Conv3D & $\sim$2,212 & --- & 145.2 & 380.5 & 32.4 \\
Full Image & 4,424 & --- & 212.8 & 516.4 & 37.8 \\
\rowcolor{IFTable}
\textbf{Ours (TF)} & \textbf{1,536} & \textbf{14.5} & \textbf{128.2} & \textbf{314.9} & \textbf{31.5} \\
\midrule
vs.\ Full Image & $-65.3\%$ & --- & $1.66\times$ & $1.64\times$ & $-16.7\%$ \\
\bottomrule
\end{tabularx}
\endgroup

\end{table*}

Main Table~\ref{tab:matched-h200} uses Qwen3-VL-8B on one NVIDIA H200,
BF16, batch size one, and the same 32 sampled frames. Timing starts from
video pixels and includes visual encoding, allocation/selection, and LLM
prefill. TTFT ends at the first generated text token; E2E ends when
generation completes. Selection is included in both totals.
Generation uses \texttt{max\_new\_tokens}=128 and permits EOS; actual
output lengths are recorded rather than forcing 128 tokens. The E2E
ratio therefore describes this EOS-terminated inference workload, not a
comparison with equal realized output lengths.

Relative to Full Image, TF uses 65.3\% fewer visual tokens, reduces TTFT
from 212.8 to 128.2 ms ($1.66\times$), reduces E2E from 516.4 to
314.9 ms ($1.64\times$), and reduces peak GPU memory from 37.8 to
31.5 GB (16.7\%). The 14.5-ms selection time applies to this 8B,
32-frame setting and is not assumed constant across input sizes.
The 32B efficiency profile is reported separately in Table~\ref{tab:eff};
it is not used to compute these matched-H200 ratios.

\subsection{Context Enrichment after Retention}
\label{app:refinement_followup}

Table~\ref{tab:refinement_primary} compares TF, a refiner restricted to
selected neighbors, and a refiner allowed to read unselected context.
Within this control, the two trainable variants use the same 65K training
examples, two epochs, architecture, initialization, and optimization and
loss settings. Only the attention mask controlling access to unselected
context differs. Extra context improves Video-MME over this study's TF
baseline by 2.22 pp; LVBench and Perception Test change by $-0.02$ and
$-0.01$ pp. Avg.(13) improves by 0.34 pp over TF and 0.80 pp over
selected-only refinement.

\begin{table*}[t]
\caption{\textbf{Context-access control, separate evaluation.} Qwen3-VL-8B uses 32 frames and about 1.5K tokens; accuracy is mean $\pm$ marginal three-seed SD. Refiners share training settings but differ in context access; differences are percentage points.}
\label{tab:refinement_primary}
\centering
\begingroup
\small
\setlength{\tabcolsep}{4pt}
\renewcommand{\arraystretch}{1.12}
\begin{tabularx}{\linewidth}{@{}lc*{4}{>{\centering\arraybackslash}X}@{}}
\toprule
\textbf{Variant} & \textbf{Training} & \textbf{Video-MME} & \textbf{LVBench} & \textbf{P.Test} & \textbf{Avg.(13)} \\
\midrule
\rowcolor{IFControl}
Selection-only (TF) & None & $65.83\!\pm\!0.07$ & $44.05\!\pm\!0.18$ & $69.25\!\pm\!0.15$ & $62.41\!\pm\!0.08$ \\
Selected-only refiner & 2 epochs & $65.80\!\pm\!0.08$ & $43.65\!\pm\!0.15$ & $68.95\!\pm\!0.12$ & $61.95\!\pm\!0.09$ \\
\rowcolor{IFTable}
With unselected context & 2 epochs & $68.05\!\pm\!0.07$ & $44.03\!\pm\!0.14$ & $69.24\!\pm\!0.11$ & $62.75\!\pm\!0.06$ \\
\midrule
Gain over TF (pp) & --- & $+2.22$ & $-0.02$ & $-0.01$ & $+0.34$ \\
Gain over selected-only (pp) & --- & $+2.25$ & $+0.38$ & $+0.29$ & $+0.80$ \\
\bottomrule
\end{tabularx}
\endgroup

\end{table*}

The SDs describe variation across seeds 42, 3407, and 114514 for each
variant separately. Paired-difference uncertainty is not available;
no significance claim is inferred from these marginal SDs. With the
training recipe held fixed, the 0.80-pp mean gap measures the effect of
allowing unselected-context access in this control. It is not an accuracy
consequence of Proposition~\ref{prop:cut}.

This context-access control uses a TF baseline of 65.83/44.05/69.25 on
Video-MME/LVBench/P.Test. It is separate from the TF task-level evaluation and
fixed-kernel reward control. The main-table refined score 68.08
and this study's 68.05 retain their own scopes; values and SDs are not
transferred between groups.\label{end:extended_results}

\section{Computed Co-selection Example}
\label{app:detail-probes}

The image-first pathway retains original source vectors while allowing
their membership to depend on temporal relationships. The computation
below isolates this distinction using a finite example with fixed source
features, frame quotas, and a declared correspondence graph.

\label{app:retention_toy}

\begin{figure*}[t]
\centering
\includegraphics[width=\linewidth]{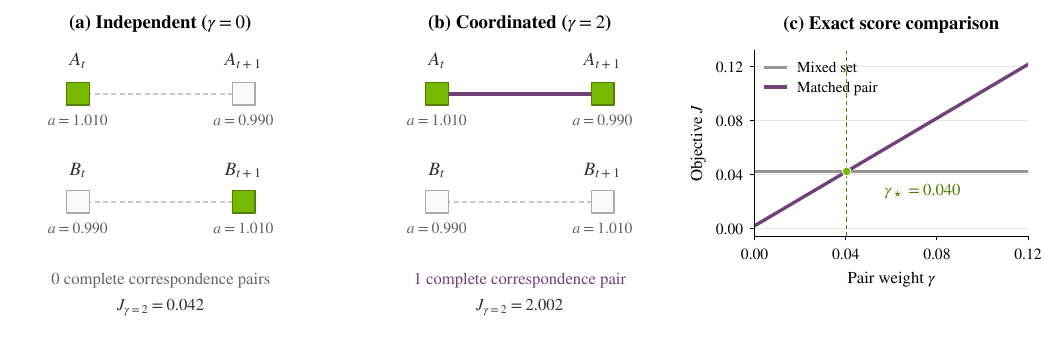}
\caption{\textbf{Computed co-selection example (synthetic).} Two frames retain one anchor each under shared kernels and quotas. Pair rewards change the selected indices; (c) reports exact objective values, not answer accuracy.}
\label{fig:retention-toy}
\end{figure*}

Fig.~\ref{fig:retention-toy} is a deterministic computation on
synthetic features, not a VLM or benchmark evaluation. The four source
vectors, grouped into two frames, are
\[
 z_0=(1,0),\quad z_1=(0,0.98),\quad
 z_2=(0.98,0),\quad z_3=(0,1).
\]
The question vector is $(1,1)/\sqrt{2}$, the relevance weight is 0.8, and
qualities and DPP jitter follow Eqs.~(\ref{eq:selector_quality_raw})--
(\ref{eq:selector_dpp_jitter}). The declared correspondence graph contains
$\{0,2\}$ and $\{1,3\}$, each with weight one; it is a synthetic input,
not a match recovered from a model. Each frame has quota one.

Independent selection ($\gamma=0$) chooses $\{0,3\}$. Our
singleton/pair greedy rule at $\gamma=2$ chooses $\{0,2\}$, retaining one
complete edge without modifying source vectors. Exact enumeration of all
four feasible sets checks these choices for this instance. For a matched
pair versus the mixed set, the break-even reward is approximately 0.040365.
Panel (c) plots the exact score functions of these fixed sets as $\gamma$
varies. Scores evaluate the stated objective, not visual detail or answer
accuracy. Agreement with enumeration here does not imply that the greedy
solver is generally exact.

The quota prevents selecting both candidates from a single frame. Every
feasible set therefore contains the same number of source vectors and
covers both frames. A score difference here reflects which cross-frame
combination is retained, rather than a change in token count or frame
coverage.

\paragraph{Matched additive control.}
Write $D(S)=\sum_t\log\det L_t[S_t,S_t]$. A fixed alternative is
\begin{equation}
\begin{aligned}
v_i&=\frac12\sum_{j:\{i,j\}\in E}w_{ij},\\
J_{\mathrm{add}}(S)&=D(S)+\gamma\sum_{i\in S}v_i.
\end{aligned}
\label{eq:toy_additive_control}
\end{equation}
The factor $1/2$ makes the full-retention graph reward equal to that of
the pair objective; each undirected edge is counted once overall. Giving
each endpoint its full incident weight would instead double that total.
Both variants use the same features, kernels, quotas, graph, $\gamma=2$,
and feasible singleton/pair moves. Only the reward changes.

In this degree-one example every node has $v_i=1/2$ and every feasible
set contains two nodes, so the additive reward is the constant $2$.
Additive selection therefore chooses $\{0,3\}$, as does the independent
selector, while pair selection chooses $\{0,2\}$. Exact enumeration checks
all four sets in Table~\ref{tab:retention-toy-exact}. The full-incident-weight
alternative also chooses $\{0,3\}$ here; it adds the constant $4$.
This isolates the difference between fixed endpoint rewards and
counterpart-dependent rewards on this constructed instance. It neither
shows that additive rewards generally choose isolated nodes nor establishes
a benchmark-accuracy advantage. Objective magnitudes across the two
different functions are not model-performance comparisons.

\begin{table}[t]
\caption{\textbf{Exact synthetic controls.} Kernels, quotas, and graph are shared, with $\gamma=2$. Objective values, not model accuracies, show how additive and joint rewards favor different source pairs.}
\label{tab:retention-toy-exact}
\centering\small
\setlength{\tabcolsep}{0pt}
\begin{tabular*}{\linewidth}{@{}l@{\extracolsep{\fill}}rrr@{}}
\toprule
Set $S$ & DPP & $J_{\mathrm{add}}$ & $J_{\mathrm{pair}}$ \\
\midrule
$\{0,2\}$ & 0.001796 & 2.001796 & 2.001796 \\
$\{0,3\}$ & 0.042161 & 2.042161 & 0.042161 \\
$\{1,2\}$ & -0.038569 & 1.961431 & -0.038569 \\
$\{1,3\}$ & 0.001796 & 2.001796 & 2.001796 \\
\bottomrule
\end{tabular*}
\end{table}
\label{end:detail-probes}

\section{Extended Related Work}
\label{app:related_work}

Image-first video encoding separates frame representation from joint
retention. Our empirical study finds that this ordering supports compact,
training-free video inputs in the evaluated Qwen systems. Image-derived
video features, conditional DPPs, and adaptive budgets are established
building blocks. We position the finding through candidate-source controls
and the retained-pair objective through fixed-kernel controls. The
literature below distinguishes candidate construction, budget assignment,
and temporal effects on feature values or subset membership.

\subsection{Closest Selection and Allocation Methods}
\label{app:closest_selection}

\paragraph{Instruction-conditioned determinantal pruning.}
CDPruner uses an instruction-dependent DPP to select informative visual
tokens \citep{zhang2025cdpruner}. Its kernel combines pairwise visual
similarity with a relevance weight for each candidate. In models without
a paired text encoder, it obtains instruction features from language-model
input embeddings and visual features after the multimodal projector.
Its evaluation also includes video understanding. Accordingly, neither
question conditioning nor using a DPP for visual-token pruning is a new
component of our method. Our within-frame kernel follows this established
quality--diversity perspective. The additional decision is whether matched
observations from different frames survive together: the correspondence
reward is outside the within-frame determinants and is earned only by a
retained endpoint pair. The matched-pool and fixed-kernel controls address
different questions. The former controls the representation presented to
the selector; the latter changes the reward while preserving the pool,
quotas, kernels, graph, and feasible move set.

\paragraph{Hierarchical video pruning.}
HieraVid organizes compression across video segments, frames, and language
model layers \citep{guo2026hieravid}. Its pipeline first reduces repeated
content through temporal organization and merging, then uses a
question-sensitive DPP to select tokens within segments, and progressively
reduces the visual sequence inside the language model. This is a close
precedent for combining temporal structure and determinantal diversity.
Our output interface is instead defined immediately before the frozen
language model: a prescribed number of original image candidates is
selected, with the same source indices applied to embeddings, positions,
and intermediate streams. The graph term specifically rewards retaining
two supplied counterparts. This is a distinction in objective and
compression location; a temporally organized DPP can already couple
choices, so cross-frame dependence in the broad sense is not exclusive
to our formulation.

\paragraph{Temporal budgeting with coordinate preservation.}
V-CAST allocates per-frame token budgets from changes in a semantic
trajectory and selects spatial evidence using diversity and feature
activation \citep{lin2026vcast}. Its retained tokens keep their original
coordinates, making it particularly relevant to an MRoPE-based interface.
Thus adaptive frame budgets and coordinate-preserving pruning should be
distinguished from our explicit co-retention reward. In our objective,
the quotas have already been computed when selection begins. At fixed
quotas, retaining one endpoint can increase the value of retaining its
counterpart; the mixed difference in main Eq.~\eqref{eq:retention_complementarity}
isolates this effect. Frame-level allocation and endpoint co-retention
therefore operate at different stages of the decision. The independent
image pool is another explicit design choice, whose effect is examined
by the candidate-pool control rather than inferred from coordinate
preservation alone.

\paragraph{Sequential subset selection before VLM compression.}
The relationship between diversity and video order predates current
VLMs. Sequential DPPs model the dependence of a selected video summary
on earlier selected subsets \citep{gong2014seqdpp}. This history is
important when interpreting temporal coordination: diversity criteria
need not treat all video observations as independent choices. Our task
uses a question-conditioned patch-token pool, integer frame quotas, and
a fixed language-model interface rather than a learned video-summary
distribution. The proposed positive edge term gives a direct numerical
meaning to retaining two source counterparts in this setting.

\begin{table*}[t]
\caption{\textbf{Closest design comparisons.} Rows compare ranking, allocation, temporal association, and retained source interfaces. They describe published mechanisms, not a common-protocol performance ranking.}
\label{tab:closest_designs}
\centering
\begingroup
\small
\setlength{\tabcolsep}{4pt}
\renewcommand{\arraystretch}{1.16}
\begin{tabularx}{\linewidth}{@{}lXXX@{}}
\toprule
Method & Main selection mechanism & Temporal organization & Retained interface \\
\midrule
CDPruner \citep{zhang2025cdpruner} & Instruction-weighted determinant & Also evaluated on video frame features & Selected visual tokens \\
HieraVid \citep{guo2026hieravid} & Segment-level DPP selection & Segmentation, merging, and layerwise reduction & Hierarchically reduced visual sequence \\
V-CAST \citep{lin2026vcast} & Curvature budgets and spatial scores & Semantic transitions guide frame budgets & Selected tokens at original coordinates \\
seqDPP \citep{gong2014seqdpp} & Sequential subset distribution & Dependence on preceding summary selections & Video summary items \\
\rowcolor{IFTable}
Ours & Within-frame determinant plus retained-pair reward & Fixed frame quotas and adjacent-frame counterparts & Original image anchors with aligned streams \\
\bottomrule
\end{tabularx}
\endgroup
\end{table*}

\subsection{Joint Retention within the Image-First Pathway}
\label{app:retention_positioning}

Our specific claim concerns the value assigned to the retained set.
A candidate can receive high question relevance without its temporal
counterpart being selected. Likewise, an adaptive quota can reserve more
slots for an important frame without determining which observation in
the next frame should accompany a retained patch. The pair term connects
these token-level decisions: its reward is conditional on membership of
both supplied endpoints, while the selected feature vectors remain
unchanged. The analysis does not require that every correspondence is
answer-relevant. Confidence and state-change filters define which edges
are offered to the selector, and the shuffled-correspondence experiment
tests the consequence of degrading that structure.

The additive control in Supplementary Sec.~\ref{app:reward_control} preserves the
same candidate features and DPP kernels. It isolates a fixed endpoint
reward from the counterpart-dependent term at the stated reward weight.
Matching total graph reward at full retention does not imply equal
regularization strength for every compressed subset. Consequently, the
reported mean difference supports this specified comparison, rather than
a guarantee against every possible weighting of an additive alternative.
The synthetic computation checks the objective and move rule on declared
inputs; it is separate from downstream task accuracy.

Table~\ref{tab:closest_designs} positions the mechanisms described in the
cited papers. The four-task Image-pool evaluation in
Sec.~\ref{app:matched_near} additionally provides matched CDPruner and
V-CAST scores at 1,536 output tokens on Qwen3-VL-8B; it does not provide
a 13-task or cross-backbone ranking against either method. No matched
HieraVid result is reported. A numerical comparison must align the
backbone, sampled frames, resolution, candidate pathway, output budget,
and evaluation protocol. For an objective-level control,
holding frame quotas fixed additionally separates allocation from the
selection rule. For a system comparison, each method's allocation and
preprocessing must instead remain part of its measured computation.

\subsection{Broader Representation and System Context}

\paragraph{Pretrained features and learned visual interfaces.}
Vision Transformers expose patch representations, while image--text
pretraining supplies transferable visual semantics
\citep{dosovitskiy2021vit,radford2021clip}.
Connecting these features to a language model can itself involve learned
compression. Flamingo uses a learned visual resampling interface, and
BLIP-2 trains a querying transformer between frozen image and language
models \citep{alayrac2022flamingo,li2023blip2}.
TokenPacker enriches coarse point queries using higher-resolution regional
features \citep{li2025tokenpacker}; LLaVA-Mini combines modality pre-fusion
with extreme visual-token compression \citep{zhang2025llavamini}.
These methods learn how to construct a compact representation. Our core
instead retains a subset of an existing frozen candidate pool. Optional
refinement changes selected feature values, but their source indices,
positional records, and output count remain fixed.

\paragraph{Temporal representation and the selection interface.}
Early video recognition explores appearance aggregation, two-stream
appearance--motion processing, and recurrent aggregation of convolutional
features \citep{karpathy2014large,simonyan2014two,donahue2015long}.
C3D and I3D learn spatiotemporal features through three-dimensional
convolutions \citep{tran2015learning,carreira2017quo}; TimeSformer and
ViViT study attention-based alternatives and spatial--temporal
factorizations \citep{bertasius2021space,arnab2021vivit}.
These designs establish several effective ways to integrate time.
VideoLLaMA 2 likewise introduces a spatial--temporal convolutional connector
for video-language modeling \citep{cheng2024videollama2}.
MotionBench further examines where temporal fusion enters a
vision--language system \citep{hong2025motionbench}, while Efficient
Motion-Aware Video MLLM exploits appearance and motion from compressed
video \citep{zhao2025ema}.
The question in our work concerns the interface exposed to a subsequent
selector. Temporally coupled features and independently encoded image
features offer different selectable units, even when derived from the
same frames. This distinction motivates a controlled candidate-source
comparison; it does not establish that temporal fusion necessarily
destroys a relevant visual cue.

\paragraph{Image-derived video representations and long-video memory.}
LLaVA-OneVision transfers across image, multi-image, and video settings;
LLaVA-Video emphasizes video instruction tuning using synthetic data
\citep{li2025onevision,zhang2025llavavideo}.
Image features already support several video-language architectures.
Video-ChatGPT connects video features to a language model, Video-LLaMA
uses a Video Q-Former to aggregate visual features, and LongVLM combines
local and global representations for long-video understanding
\citep{maaz2024videochatgpt,zhang2023videollama,weng2024longvlm}.
Koala conditions learnable spatiotemporal queries on sparse key frames
\citep{tan2024koala}. MA-LMM processes video online and maintains a memory
bank for long-term reference \citep{he2024malmm}.
These approaches address temporal coverage through aggregation, learned
queries, or memory. They make clear that using image-derived features
for video is established. Our study asks a narrower question: can a frozen
image pathway supply independently addressable candidates whose joint
retention yields a compact input for an existing frozen VLM? The output
units remain selected source anchors, with learned context refinement
treated as an optional extension.

\paragraph{Context extension and extreme compression.}
LongVA extends the language backbone's context to accommodate long visual
sequences \citep{zhang2025longva}. VideoChat-Flash uses hierarchical
clip-to-video compression \citep{li2026videochatflash}.
More recent extreme compression combines learned progressive token
compression with question-conditioned frame selection
\citep{zhang2026onetoken}.
These approaches expand temporal coverage by increasing the readable
context or changing how observations are summarized. Our fixed output
budget instead constrains the source anchors passed to an existing frozen
VLM. The distinction matters when interpreting denser sampling: a fixed
output length bounds the retained visual sequence, while the image
encoder must still process the sampled frames.

\paragraph{Adaptive sampling, resolution, and frame budgets.}
Temporal coverage competes with spatial detail when the visual budget is
limited. SlowFast-LLaVA combines a low-rate pathway that retains spatial
detail with a high-rate pathway that supplies broader temporal coverage
\citep{xu2024slowfastllava}. Q-Frame uses query-aware frame selection and
multi-resolution adaptation \citep{zhang2025qframe}.
These strategies change which frames or resolutions are encoded. Our
allocator instead assigns token quotas to already encoded frames, after
which the selector chooses individual source tokens. The decisions are
related, but their costs and retained information differ: increasing
input sampling can still increase image-encoding cost even when the
output-token budget is fixed. The component analysis evaluates the
contribution of our frame-budget allocator.

\paragraph{Token reduction before and inside the language model.}
Token reduction also has a history inside visual backbones. EViT
reorganizes image tokens according to their contribution to recognition;
ToMe merges similar tokens using a lightweight matching procedure
\citep{liang2022evit,bolya2023tome}.
The location of compression determines which computation it can avoid.
FastV uses attention from early language-model layers to prune visual
tokens in subsequent layers \citep{chen2024fastv}. SparseVLM uses
text-guided visual-token sparsification, adaptive layerwise pruning, and
token recycling \citep{zhang2025sparsevlm}. VisionZip reduces the visual
sequence before the language model using informative tokens
\citep{yang2025visionzip}.
DyCoke combines temporal token merging with dynamic KV-cache reduction
during decoding \citep{tao2025dycoke}. HoliTom addresses both outer-LLM
spatiotemporal redundancy and inner-LLM token redundancy
\citep{shao2025holitom}.
These methods reduce different portions of inference. Consequently,
matching a final token count does not by itself match visual-encoding
cost, early-layer computation, or decoding cost. Our fixed output budget
specifies the visual sequence passed to the language model; allocation,
correspondence construction, and selection remain part of inference.

\paragraph{Informative selection, diversity, and source retention.}
LLaVA-PruMerge identifies important image tokens and merges additional
information into the retained representation \citep{shang2025prumerge}.
DivPrune formulates visual-token pruning through max--min diversity
\citep{alvar2025divprune}. AIM combines similarity-based merging before
the LLM with progressive importance-based pruning inside it
\citep{zhong2025aim}. These methods make redundancy and diversity explicit
design targets. They also illustrate why a shared output budget does not
fully specify a compression method: retained source vectors, merged
vectors, and learned summaries can have different feature and position
interfaces. We preserve this distinction in the candidate-pool control
and treat the DPP term as an established diversity mechanism, rather than
the novelty by itself.

\paragraph{Spatiotemporal compression and adaptive retention.}
Recent video-specific compressors already reason beyond independent
spatial pruning. FastVID combines dynamic temporal segmentation with
density-based spatiotemporal pruning \citep{shen2025fastvid}.
VidCom$^2$ adapts compression intensity to frame uniqueness through a
frame-compression adjustment mechanism \citep{liu2025vidcom2}.
FlashVID combines attention--diversity token selection with tree-based
spatiotemporal merging \citep{fan2026flashvid}.
Unified spatiotemporal compression jointly considers frame and token
retention and recycles discarded-token information under very low
retention ratios \citep{du2026unified}.
Thus, adaptive allocation and cross-frame redundancy reduction are
established ideas. Our distinction is the combination of an independent
image candidate pathway with source-preserving coordinated selection.
The matched-pool comparison tests the selector on a common representation,
while the pool-by-selector comparison measures how its gain changes with
the candidate pathway.

\paragraph{Diversity and correspondence play different roles.}
Determinantal point processes provide an established way to combine item
quality with diversity \citep{kulesza2012determinantal}.
We use a DPP-derived log-determinant term to discourage redundant
directions within each frame. Adjacent-frame correspondence enters a
separate pair reward over selected anchors. It couples retention
decisions across time without averaging the selected embeddings during
selection. Similarity is therefore used in two ways: to avoid redundant
within-frame selections and to favor specified cross-frame relationships.
The resulting objective extends the within-frame diversity criterion;
it does not inherit an exact optimizer or a global guarantee from a
standard DPP formulation. The absent- and shuffled-correspondence
controls test whether the supplied relationships help in practice.

\paragraph{Source alignment and optional context refinement.}
Modern VLMs attach structure to visual tokens beyond their final
embedding vectors. Qwen models use multimodal positional encoding, and
Qwen3-VL also introduces DeepStack features
\citep{wang2024qwen2vl,bai2025qwen3vl}.
Our selected source indices gather the corresponding embeddings,
positions, and intermediate features together. This defines a
consistent anchor interface; it is not a claim that merged features
cannot preserve useful information. Selection determines which anchors
occupy the output budget, while optional refinement allows those anchors
to read additional neighboring context and receive residual updates.
Both operations retain the same output count, but only the latter
requires learned refinement. This separates the training-free core from
the context-access experiment and helps interpret its modest aggregate
benefit. The conditional-prediction and fixed-operator analyses explain
these interfaces under stated assumptions; the benchmark comparisons
provide the empirical evidence for accuracy.

\paragraph{Layer-dependent evidence and system efficiency.}
Analysis of token information shows that pruning behavior varies with
language-model depth, model capacity, and visual task
\citep{wang2026informationhorizon}. This concerns representations inside
the language model, rather than a direct measurement of what a Conv3D
patch embedding removes. NVILA jointly studies visual-token compression
and broader training and deployment optimizations \citep{liu2025nvila}.
Together, these perspectives motivate separating candidate construction,
compression location, and system cost. Our accuracy controls isolate
selected parts of this pipeline; the reported timing profiles retain the
hardware and measurement qualifications in Supplementary
Sec.~\ref{app:extended_results}.\label{end:related_work}

\section{Additional Experimental Results}
\label{app:additional-results}

Table~\ref{tab:general-exact} reports the exact scores for the seven
benchmarks discussed in main-paper Sec.~\ref{sec:cross-benchmark-generalization}, together with the
macro-average over all 13 benchmarks. Each benchmark receives equal weight:
$\mathrm{Avg.}(13)$ is computed directly from the 13 individual benchmark
scores rather than by averaging the six- and seven-benchmark group means.

\begin{table*}[t]
\caption{\textbf{Additional seven-benchmark results.} Scores are accuracy (\%); Avg.(13) averages all 13 benchmarks. +ref. adds refinement and $\dagger$ marks the separate 8B TF evaluation; bold/underlined values are best/second-best among budget-matched methods.}

\label{tab:general-exact}
\centering

\begingroup
\small
\setlength{\tabcolsep}{1.4pt}
\renewcommand{\arraystretch}{1.00}

\begin{tabularx}{\linewidth}
{@{}l*{8}{>{\centering\arraybackslash}X}@{}}
\toprule
Method
& EGO
& \shortstack{VINO\\text}
& VideoAds
& LVBench
& MVBench
& STAR
& IntentQA
& \shortstack{Avg.(13)\\$\uparrow$} \\
\midrule

\multicolumn{9}{l}{\textit{Qwen2-VL-7B}} \\[-1pt]

Video/Conv3D
& 63.00
& 65.00
& 51.82
& 39.35
& 63.19
& 76.06
& 94.85
& 57.31 \\

Image
& 70.00
& 65.50
& 58.64
& 41.94
& 60.97
& 75.07
& 93.44
& 59.26 \\

Uniform
& 61.35
& 58.74
& 46.52
& 36.15
& 60.18
& 71.44
& 91.22
& 53.67 \\

FlashVID
& 62.18
& 59.83
& \underline{49.46}
& 38.08
& 62.33
& 73.37
& 93.14
& 55.41 \\

VidCom$^{2}$
& 63.05
& 60.48
& 49.12
& 38.36
& \underline{63.22}
& 74.11
& \underline{94.42}
& 56.06 \\

FastVID
& 63.96
& 64.05
& 45.88
& 38.74
& 62.05
& 74.82
& 94.11
& 56.24 \\

VisionZip
& \underline{64.95}
& \textbf{64.54}
& 48.15
& \underline{38.75}
& \textbf{64.12}
& \underline{74.96}
& \textbf{95.05}
& \underline{56.88} \\

\rowcolor{IFTable}
\textbf{Ours + ref.}
& \textbf{71.33}
& \underline{64.33}
& \textbf{59.55}
& \textbf{42.33}
& 59.43
& \textbf{76.18}
& 92.77
& \textbf{59.43} \\

\midrule

\multicolumn{9}{l}{\textit{Qwen3-VL-8B}} \\[-1pt]

Video/Conv3D
& 61.00
& 65.00
& 55.00
& 39.35
& 68.19
& 73.38
& 81.26
& 59.49 \\

Image
& 73.00
& 68.00
& 61.82
& 43.87
& 66.25
& 75.49
& 82.67
& 62.58 \\

Uniform
& 59.25
& 63.14
& 53.68
& 37.15
& 64.72
& 70.36
& 79.15
& 56.93 \\

FastV
& 61.15
& 65.25
& 55.34
& 38.45
& 66.18
& 72.15
& 80.35
& 58.40 \\

SparseVLM
& 61.45
& 65.88
& 55.72
& 39.12
& 66.55
& 72.48
& 80.75
& 58.81 \\

DyCoke
& 61.85
& 66.35
& 56.15
& 39.45
& 66.92
& 72.85
& 81.12
& 59.21 \\

FlashVID
& 62.05
& \underline{67.92}
& 56.40
& 39.95
& 67.42
& 72.84
& 81.08
& 59.62 \\

VidCom$^{2}$
& 60.08
& 65.46
& 55.95
& 38.02
& \underline{68.38}
& 73.47
& 82.02
& 59.63 \\

FastVID
& 62.94
& 64.55
& 57.78
& 41.02
& \textbf{68.84}
& 72.21
& 81.46
& 59.86 \\

VisionZip
& 61.95
& 65.56
& 54.51
& 39.96
& 68.25
& 73.04
& 81.56
& 59.45 \\

\rowcolor{IFControl}
\textbf{Ours (TF)}$^\dagger$
& \textbf{73.43}
& \textbf{68.12}
& \textbf{62.55}
& \underline{42.26}
& 66.91
& \textbf{75.94}
& \textbf{83.05}
& \underline{62.64} \\

\rowcolor{IFTable}
\textbf{Ours + ref.}
& \underline{73.13}
& \underline{67.92}
& \underline{62.15}
& \textbf{44.03}
& 66.71
& \underline{75.64}
& \underline{82.77}
& \textbf{62.75} \\

\midrule

\multicolumn{9}{l}{\textit{Qwen3-VL-32B}} \\[-1pt]

Video/Conv3D
& 68.00
& 73.50
& 61.36
& 42.90
& 71.39
& 75.35
& 87.82
& 64.20 \\

Image
& 75.00
& 77.50
& 65.91
& 44.19
& 67.78
& 78.87
& 88.29
& 66.09 \\

Uniform
& 63.15
& 68.45
& 57.12
& 40.85
& 65.34
& 72.15
& 84.12
& 59.94 \\

FlashVID
& 55.06
& 59.92
& 51.42
& 38.65
& 60.08
& 54.81
& 80.37
& 54.32 \\

VidCom$^{2}$
& \underline{67.92}
& 70.56
& 60.41
& \underline{44.24}
& \textbf{71.16}
& \underline{76.36}
& \underline{87.77}
& \underline{63.86} \\

FastVID
& 66.95
& 71.44
& \underline{61.88}
& 44.14
& \underline{70.75}
& 75.36
& 87.64
& 63.57 \\

VisionZip
& 65.94
& \underline{72.08}
& 60.39
& 42.63
& 69.78
& 75.80
& 87.18
& 63.29 \\

\rowcolor{IFTable}
\textbf{Ours + ref.}
& \textbf{75.33}
& \textbf{77.58}
& \textbf{66.24}
& \textbf{44.32}
& 68.15
& \textbf{78.96}
& \textbf{88.43}
& \textbf{66.28} \\

\bottomrule
\end{tabularx}

\endgroup
\end{table*}\label{end:additional-results}

\section{Extended Ablations and Training Details}
\label{app:ablations}
\label{sec:extended_ablations}

This section provides additional experiments that isolate the source of our
gains, document the training of the temporal refiner, and examine robustness
to correspondence structure, candidate-frame coverage, and visual-token
budget.

\subsection{Training-Free and Complete-Pathway Summary}
\label{app:training-fair}

Most compression baselines in the main comparison are training-free, whereas
our complete pathway additionally contains a learned temporal refiner.
Disabling refinement leaves a fully training-free \emph{selection-only}
pathway. Table~\ref{tab:training-fair} summarizes both variants at comparable
visual-token budgets. The 8B TF task-level and complete-pathway evaluations
are separate groups, so their difference is not a paired refinement effect.
The matched context-access control is reported in
Table~\ref{tab:refinement_primary}.

\begin{table*}[t]
\caption{\textbf{Training-free and complete-pathway summary.} TF requires no compressor training; Selection disables refinement. Avg.(13) is mean accuracy (\%), $\Delta$ is gain over the strongest prior TF baseline (pp), and $\dagger$ marks the separate 8B TF evaluation.}
\label{tab:training-fair}
\centering

\begingroup
\small
\setlength{\tabcolsep}{2.2pt}
\renewcommand{\arraystretch}{1.02}

\begin{tabularx}{\linewidth}
{@{}lc*{4}{>{\centering\arraybackslash}X}@{}}
\toprule
Method
& Train
& \shortstack{Visual tokens $\downarrow$}
& \shortstack{Avg.(6) $\uparrow$}
& \shortstack{Avg.(13) $\uparrow$}
& \shortstack{$\Delta$ prior TF $\uparrow$} \\
\midrule

\multicolumn{6}{l}{\textit{Qwen2-VL-7B}} \\[-1pt]

Full Image
& --
& 5,561
& 50.81
& 59.26
& -- \\

VisionZip
& TF
& 1,532
& 48.15
& 56.88
& 0.00 \\

\rowcolor{IFControl}
\textbf{Ours: selection}
& \textbf{TF}
& 1,536
& 50.97
& 59.37
& \textbf{+2.49} \\

\rowcolor{IFTable}
\textbf{Ours: + refiner}
& Learned
& 1,536
& \textbf{51.11}
& \textbf{59.43}
& \textbf{+2.55} \\

\midrule

\multicolumn{6}{l}{\textit{Qwen3-VL-8B}} \\[-1pt]

Full Image
& --
& 4,424
& 57.08
& 62.58
& -- \\

FastVID
& TF
& 1,493
& 54.90
& 59.86
& 0.00 \\

\rowcolor{IFControl}
\textbf{Ours: selection}$^\dagger$
& \textbf{TF}
& 1,536
& 57.01
& 62.64
& \textbf{+2.78} \\

\rowcolor{IFTable}
\textbf{Ours: + refiner}
& Learned
& 1,535
& \textbf{57.23}
& \textbf{62.75}
& \textbf{+2.89} \\

\midrule

\multicolumn{6}{l}{\textit{Qwen3-VL-32B}} \\[-1pt]

Full Image
& --
& 4,424
& 60.28
& 66.09
& -- \\

VidCom$^{2}$
& TF
& 1,504
& 58.62
& 63.86
& 0.00 \\

\rowcolor{IFControl}
\textbf{Ours: selection}
& \textbf{TF}
& 1,535
& 60.33
& 66.22
& \textbf{+2.36} \\

\rowcolor{IFTable}
\textbf{Ours: + refiner}
& Learned
& 1,535
& \textbf{60.44}
& \textbf{66.28}
& \textbf{+2.42} \\

\bottomrule
\end{tabularx}

\endgroup
\end{table*}

\textbf{Training-free selection accounts for the primary gain.}
The selection-only variant reaches $\mathrm{Avg.}(13)$ scores of 59.37,
62.64, and 66.22 on Qwen2-VL-7B, Qwen3-VL-8B, and Qwen3-VL-32B,
respectively. These exceed the strongest prior training-free compression
baselines by \textbf{2.49}, \textbf{2.78}, and \textbf{2.36} percentage
points. They also match or slightly exceed the corresponding Full Image
references while using only approximately 1.5K visual tokens. Thus the main
improvement does not depend on additional training.

\textbf{Temporal refinement provides a targeted complement.}
The complete-pathway $\mathrm{Avg.}(13)$ scores exceed the selection-only
summaries by
\textbf{+0.06}, \textbf{+0.11}, and \textbf{+0.06} points across the three
backbones. These summary gaps are distinct from the matched context-access
effects in Table~\ref{tab:refinement_primary}, where Video-MME improves while
LVBench and Perception Test decrease slightly. Coordinated selection
determines which evidence occupies the compact sequence; refinement adds
temporal context after the output anchors are fixed.

\subsection{Allocation and Coordination Ablation}
\label{app:component_objective_ablation}

This selection-only study uses Qwen3-VL-8B with approximately 1,535 visual
tokens. It changes allocation and then coordination within this evaluation
group; no learned refiner is used. Table~\ref{tab:component_objective_ablation}
reports the allocation/coordination ablation, distinct from the fixed-kernel reward control.

\begin{table*}[t]
\caption{\textbf{Allocation and coordination, without refinement.} Qwen3-VL-8B retains about 1.5K tokens as allocation and coordination are enabled successively. Scores are accuracy (\%); this ablation is separate from the fixed-kernel reward control.}
\label{tab:component_objective_ablation}
\centering
\begingroup
\small
\setlength{\tabcolsep}{4pt}
\renewcommand{\arraystretch}{1.12}
\begin{tabularx}{\linewidth}{@{}cc*{5}{>{\centering\arraybackslash}X}@{}}
\toprule
\makecell{\textbf{Dynamic}\\\textbf{allocation}} &
\makecell{\textbf{Cross-frame}\\\textbf{coordination}} &
\textbf{VideoAds} & \textbf{Video-MME} & \textbf{LVBench} &
\textbf{P.Test} & \textbf{Avg.(13)} \\
\midrule
\xmark & \xmark & 55.45 & 63.85 & 39.80 & 67.85 & 60.12 \\
\cmark & \xmark & 60.30 & 65.40 & 42.15 & 68.30 & 61.74 \\
\rowcolor{IFTable}
\cmark & \cmark & \textbf{62.55} & \textbf{66.25} & \textbf{42.26} & \textbf{68.65} & \textbf{62.64} \\
\bottomrule
\end{tabularx}
\endgroup
\end{table*}

Allocation raises Avg.(13) by 1.62 pp; coordination adds 0.90 pp.
These are successive changes along this ablation path, not a factorial
estimate of super-additivity. Refinement is analyzed using the separate
context-access comparison in Table~\ref{tab:refinement_primary}.

\subsection{Refiner Training Details}
\label{app:refiner_training}

For the main complete-pathway evaluation, the temporal refiner consists of a bottleneck
cross-frame attention block and a gated feed-forward network, introducing
approximately 18M trainable parameters. The vision encoder, language model,
and coordinated selector remain frozen throughout training. We train only the
refiner for two epochs on 8 NVIDIA H200 GPUs for 12 hours (approximately 96 GPU-hours)
using AdamW, a learning rate of $2\times10^{-5}$, and a batch size of 64.
For knowledge distillation, we use temperature $\tau=2.0$ and weight
$\beta=1.0$. These settings apply to the complete-pathway evaluation;
the separate context-access control uses the matched training conditions
specified in Table~\ref{tab:refinement_primary} and
Supplementary Sec.~\ref{app:refinement_followup}.

\paragraph{Training data and decontamination.}
We construct a 65K-example training subset from the Video-ChatGPT training
split. To reduce train--test overlap, we filter this subset against the
evaluation prompts and available metadata of all 13 downstream benchmarks
using both exact n-gram matching and a semantic-similarity filter, removing
every sample flagged by either criterion.

\subsection{System-Level Ablations}
\label{app:system_ablation}

We next examine three system-level choices on Qwen2-VL-7B:
cross-frame correspondence, candidate-frame coverage, and the visual-token
budget. All comparisons use the same 13-benchmark macro-average. Table~\ref{tab:hyper_ablation} summarizes these system-level ablations.

\begin{table*}[t]
\caption{\textbf{System-level ablations on Qwen2-VL-7B.} Rows vary selection, input frames, or output tokens. Avg.(13) is mean accuracy (\%), and $\Delta$ is the change from the highlighted default in percentage points.}
\label{tab:hyper_ablation}
\centering

\begingroup
\small
\setlength{\tabcolsep}{2.0pt}
\renewcommand{\arraystretch}{1.02}

\begin{tabularx}{\linewidth}
{@{}ll*{5}{>{\centering\arraybackslash}X}@{}}
\toprule
\textbf{Dimension} &
\textbf{Setting} &
\shortstack{\textbf{Candidate}\\\textbf{frames}} &
\shortstack{\textbf{Token}\\\textbf{budget}} &
\shortstack{\textbf{Actual}\\\textbf{tokens} $\downarrow$} &
\shortstack{\textbf{Avg.(13)}\\$\uparrow$} &
\shortstack{\textbf{$\Delta$}\\\textbf{(pp)}} \\
\midrule

Selection
& Independent
& 32 & 1,536 & 1,534.9 & 58.50 & $-0.93$ \\

Selection & Shuffled correspondence
& 32 & 1,536 & 1,534.9 & 57.53 & $-1.90$ \\

\rowcolor{IFTable}
Selection & \textbf{Coupled (Ours)}
& \textbf{32} & \textbf{1,536}
& \textbf{1,534.9} & \textbf{59.43} & \textbf{0.00} \\

\midrule

Candidate frames
& 16 frames
& 16 & 1,536 & 1,517.7 & 57.97 & $-1.46$ \\

Candidate frames & 24 frames
& 24 & 1,536 & 1,530.1 & 58.85 & $-0.58$ \\

\rowcolor{IFTable}
Candidate frames & \textbf{32 frames (Ours)}
& \textbf{32} & \textbf{1,536}
& \textbf{1,534.9} & \textbf{59.43} & \textbf{0.00} \\

\midrule

Token budget
& 768 tokens
& 32 & 768 & 767.8 & 56.63 & $-2.80$ \\

\rowcolor{IFTable}
Token budget & \textbf{1,536 tokens (Ours)}
& \textbf{32} & \textbf{1,536}
& \textbf{1,534.9} & \textbf{59.43} & \textbf{0.00} \\

Token budget & 3,072 tokens
& 32 & 3,072 & 3,035.3 & 61.11 & $+1.68$ \\

\bottomrule
\end{tabularx}

\endgroup

\end{table*}

\textbf{Correct cross-frame correspondence matters.}
Ignoring cross-frame correspondence reduces $\mathrm{Avg.}(13)$ from 59.43
to 58.50. Randomly shuffling the correspondences decreases it further to
57.53, suggesting that incorrect pair structure can be actively harmful
rather than merely uninformative.

\textbf{Broader candidate coverage improves performance.}
Increasing the candidate pool from 16 to 24 frames raises
$\mathrm{Avg.}(13)$ from 57.97 to 58.85, and increasing it from 24 to 32
frames gives a further 0.58-point improvement. The monotonic increase over
the evaluated range is consistent with broader temporal coverage.

\textbf{The output budget controls the accuracy--compression trade-off.}
Reducing the budget from 1,536 to 768 tokens lowers $\mathrm{Avg.}(13)$ by
2.80 points, whereas increasing it to 3,072 tokens raises the score by
1.68 points. We use approximately 1.5K tokens in the main experiments because
it matches the operating regime of prior compression baselines while retaining
a substantial reduction relative to the Full Image pathway.

\subsection{32B Efficiency Profile and Breakdown}
\begin{table*}[t]
\caption{\textbf{Qwen3-VL-32B efficiency profile.}
Ratios compare each method's end-to-end latency with Full Image.}

\label{tab:efficiency}
\centering
\fontsize{8.8}{10.6}\selectfont
\setlength{\tabcolsep}{1.8pt}
\renewcommand{\arraystretch}{1.05}
\begin{tabularx}{\linewidth}{@{}l*{9}{>{\centering\arraybackslash}X}@{}}
\toprule
Method &
\shortstack{Visual\\tokens $\downarrow$} &
\shortstack{Selection\\(ms) $\downarrow$} &
\shortstack{TTFT\\(ms) $\downarrow$} &
\shortstack{E2E\\(ms) $\downarrow$} &
\shortstack{Through-\\put\\(q/s) $\uparrow$} &
\shortstack{Memory\\(GB) $\downarrow$} &
\shortstack{Avg.(13)\\$\uparrow$} &
\shortstack{$\Delta$ Image\\(pp) $\uparrow$} &
\shortstack{Image/E2E\\ratio $\uparrow$} \\
\midrule
Video/Conv3D & 2212.0 & --    & 454.0 & 508.5 & 1.97 & 67.97 & 64.20 & -1.89  & 1.79$\times$ \\
\rowcolor[RGB]{250,250,250}Image        & 4424.0 & --    & 849.1 & 911.2 & 1.10 & 69.04 & 66.09 &  0.00  & 1.00$\times$ \\
\midrule
FlashVID     & 1527.0 & 43.59 & \textbf{297.7} & \textbf{330.5} & \textbf{3.03} & \textbf{67.52} & 54.32 & -11.77 & \textbf{2.76$\times$} \\
VidCom$^2$   & 1504.0 & \textbf{3.28} & 348.4 & 403.3 & 2.48 & 67.88 & 63.86 & -2.23 & 2.26$\times$ \\
FastVID      & \textbf{1485.0} & 15.38 & 356.7 & 413.8 & 2.42 & 67.81 & 63.58 & -2.51 & 2.20$\times$ \\
VisionZip    & 1521.0 & 57.16 & 409.2 & 459.4 & 2.18 & 67.82 & 63.29 & -2.80 & 1.98$\times$ \\
\rowcolor[RGB]{235,225,227}\textbf{Ours + refiner}& 1535.0 & 14.52 & 365.8 & 421.8 & 2.37 & 67.65 & \textbf{66.28} & \textbf{+0.19} & 2.16$\times$ \\
\bottomrule
\end{tabularx}
\label{tab:eff}

\end{table*}

The Qwen3-VL-32B profile appears in Table~\ref{tab:eff}. The following figure gives its latency breakdown. This profile does not establish speedups on a common accelerator and is separate from the matched single-H200 evaluation.
\label{app:efficiency_breakdown}

\begin{figure*}[t]
  \centering
  \includegraphics[width=\linewidth]{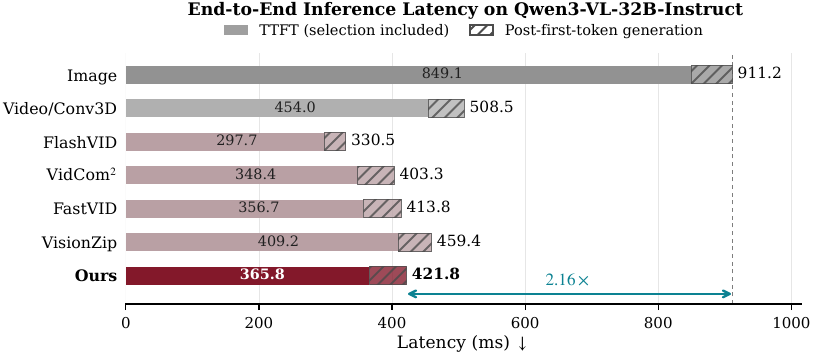}
  
  \caption{\textbf{32B latency breakdown.} First-token time, including selection, and subsequent generation sum to end-to-end latency. This Qwen3-VL-32B profile is separate from the matched single-H200 evaluation.}
  \label{fig:latency-breakdown}
  
\end{figure*}

The latency reduction primarily occurs before autoregressive decoding: TTFT
falls from 849.1\,ms for Image to 365.8\,ms for Ours, while post-first-token
generation changes from 62.1\,ms to 56.0\,ms. The selection overhead is
included in TTFT and is offset by the shorter visual prefill.

\begin{figure*}[t]
  \centering
  \includegraphics[width=\linewidth]{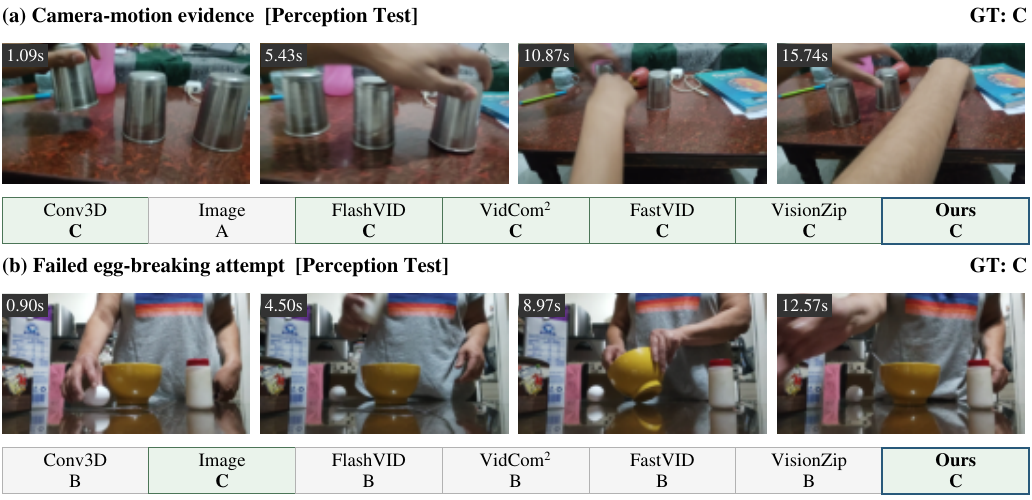}
  
  \caption{\textbf{Additional local and relational examples.} Qwen3-VL-8B cases cover camera motion and an egg-breaking attempt. GT denotes the ground-truth option and green marks correct predictions; timestamped frames show input context, not token masks.}
  \label{fig:qualitative-success-1}
  
\par\smallskip
\begingroup
\makeatletter
\def\@captype{table}
\makeatother
\captionsetup{type=table}
\caption{\textbf{Complete method predictions for the main-paper video cases.} Entries are parsed answer options on Qwen3-VL-8B; GT is the ground truth. Rows follow main Fig.~\ref{fig:qualitative-main}(a--d); bold entries mark correct predictions.}
\label{tab:main_case_predictions}
\centering
\begingroup
\small
\setlength{\tabcolsep}{0pt}
\renewcommand{\arraystretch}{1.13}
\begin{tabular*}{\linewidth}{@{\extracolsep{\fill}}lcccccccc@{}}
\toprule
Case & GT & Conv3D & Image & FlashVID & VidCom$^2$ & FastVID & VisionZip & Ours \\
\midrule
Head movement & C & B & B & B & B & B & B & \textbf{C} \\
Shoe-lace interaction & B & A & \textbf{B} & A & A & A & A & \textbf{B} \\
Object counting & B & C & A & A & C & C & C & \textbf{B} \\
Subject binding & E & C & \textbf{E} & C & C & C & C & C \\
\bottomrule
\end{tabular*}
\endgroup
\endgroup
\end{figure*}

\begin{figure*}[t]
  \centering
  \includegraphics[width=\linewidth]{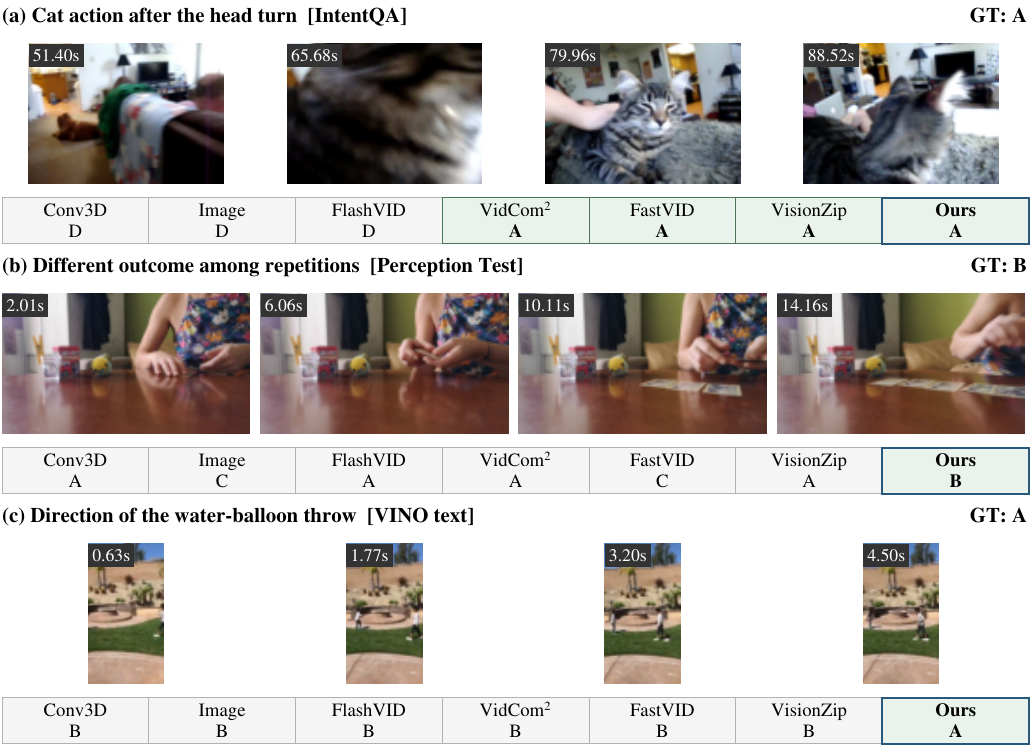}
  
  \caption{\textbf{Additional temporal examples.} Qwen3-VL-8B cases cover late evidence, repeated outcomes, and throw direction. GT denotes ground truth and green marks correct predictions; timestamped frames show input context, not token masks.}
  
  \label{fig:qualitative-success-2}
\end{figure*}

\subsection{Multi-Budget Accuracy--Latency Trade-off}
\label{app:multibudget-pareto}

Main Fig.~\ref{fig:multibudget-main} summarizes the accuracy--latency trade-off
across the three evaluated budgets and backbones. The main experiments use an approximately 1.5K visual-token budget to match
prior compression methods. We further sweep the output budget over
$\{768,1536,3072\}$ tokens on Qwen2-VL-7B, Qwen3-VL-8B, and
Qwen3-VL-32B. The sweep reports target budgets; capacities can limit actual output counts. We evaluate visual-token counts,
end-to-end inference latency, and the macro-average over all 13 benchmarks.

\textbf{The advantage persists across target budgets.}
At 768 tokens, our method obtains $\mathrm{Avg.}(13)$ scores of 56.63, 60.15,
and 64.15 on Qwen2-VL-7B, Qwen3-VL-8B, and Qwen3-VL-32B, respectively.
These exceed the strongest prior compression baseline at the same budget by
2.83, 3.35, and 3.65 points. At the 1.5K budget, the corresponding gains are
2.55, 2.89, and 2.42 points, and at 3K they remain 2.66, 2.75, and 2.75
points. The relative advantage therefore persists over a $4\times$ range of
output budgets.

\textbf{A smaller coordinated representation can outperform larger baselines.}
The 1.5K operating point of our method also outperforms every prior method
evaluated at the 3K budget. Ours reaches 59.43, 62.75, and 66.28, compared
with the strongest 3K baselines at 58.45, 61.50, and 65.10, respectively.
The corresponding end-to-end latencies are 215.5 vs.\ 405.2\,ms,
260.5 vs.\ 455.2\,ms, and 421.8 vs.\ 685.2\,ms.

\textbf{Returns diminish beyond 1.5K tokens.}
Increasing our budget from 768 to approximately 1.5K tokens improves
$\mathrm{Avg.}(13)$ by 2.80, 2.60, and 2.13 points across the three
backbones. Doubling the budget again to approximately 3K yields smaller gains
of 1.68, 1.50, and 1.57 points while increasing latency. This supports the
1.5K setting as a practical accuracy--efficiency operating point.\label{end:ablations}

\section{Qualitative Analysis}
\label{app:qualitative}

Figs.~\ref{fig:qualitative-success-1} and~\ref{fig:qualitative-success-2} add five Qwen3-VL-8B examples to the four cases in main Fig.~\ref{fig:qualitative-main}. Table~\ref{tab:main_case_predictions} supplies complete method predictions for those four cases. Together, the examples illustrate the kinds of visual evidence that matter under a compact budget rather than estimating how frequently a particular behavior occurs.

\paragraph{Visualization protocol.}
For each example, we show four representative frames from the 32-frame input together with their timestamps and the parsed multiple-choice prediction of every evaluated method. Green cells indicate correct predictions, and our method is additionally outlined for identification. The displayed frames are selected only for visualization; they should not be interpreted as a reconstruction of any method's internal token mask or retained-token support. We include cases where other compression methods also succeed, rather than restricting the visualization to exclusive wins.

\paragraph{What the successful cases illustrate.}
The examples suggest several recurring evidence patterns. Fine-grained local change matters when the answer depends on a small motion rather than a dominant object or pose (main Fig.~\ref{fig:qualitative-main}a). Relational evidence matters when two observations must be interpreted jointly, as in bilateral shoe-lace interaction, the failed egg-breaking attempt, and the water-balloon direction (main Fig.~\ref{fig:qualitative-main}b; Figs.~\ref{fig:qualitative-success-1}b and~\ref{fig:qualitative-success-2}c). Long-range coverage is useful when the decisive evidence occurs late or when outcomes must be compared across repeated actions (Fig.~\ref{fig:qualitative-success-2}a,b). Finally, repeated views should not be confused with distinct evidence, as illustrated by the object-counting example in main Fig.~\ref{fig:qualitative-main}c. These examples are consistent with the intended role of coordinated selection: allocate the limited output budget to complementary source evidence rather than to individually salient but redundant observations.

\paragraph{Failure analysis.}
The case in main Fig.~\ref{fig:qualitative-main}d stresses long-range subject binding rather than local appearance recognition. The model must associate a late action with the correct parrot among multiple visually similar birds, foliage, and cage structure. Full Image answers correctly, while all compressed pathways in Table~\ref{tab:main_case_predictions} choose the same distractor. These predictions identify a failure case; they do not establish which internal evidence was lost. The case motivates further evaluation of instance identity across long temporal spans. The shared wrong answer does not distinguish missing action evidence from incorrect subject association. Tracing retained inputs and testing interventions could separate these causes.\label{end:qualitative}

\proofspacing
\section{Coordinated Selection: Objective and Analysis}
\label{app:selection}

This supplementary section gives the details behind the determinantal diversity term,
the adjacent-frame pair reward, and the feasibility of the greedy selection
procedure in main Section~\ref{sec:coordinated_selection}. The analysis concerns
the selection operators themselves and does not assume that the resulting
subset is globally optimal.

\paragraph{Determinantal marginal gain.}
Consider one frame and suppress the frame index. Let
$L\succeq 0$ be its DPP kernel, let $U$ be the currently selected candidates
from that frame, and let $i\notin U$. Whenever $L_{UU}$ is positive definite,
the determinant of the enlarged principal block satisfies
\begin{equation}
\det L_{U\cup\{i\},U\cup\{i\}}
=
\det L_{UU}
\left(
L_{ii}-L_{iU}L_{UU}^{-1}L_{Ui}
\right).
\label{eq:app_dpp_schur}
\end{equation}
Consequently,
\begin{align}
\Delta_i^{\mathrm{DPP}}(U)
&=
\log\det L_{U\cup\{i\},U\cup\{i\}}
-\log\det L_{UU}
\nonumber\\
&=
\log\!\left(
L_{ii}-L_{iU}L_{UU}^{-1}L_{Ui}
\right).
\label{eq:app_dpp_marginal}
\end{align}

\begin{proof}
After ordering the indices with those in $U$ first and $i$ last, the enlarged
principal block is
\[
L_{U\cup\{i\},U\cup\{i\}}
=
\begin{bmatrix}
L_{UU} & L_{Ui}\\
L_{iU} & L_{ii}
\end{bmatrix}.
\]
Since $L_{UU}$ is invertible, the block-determinant identity gives
\[
\det
\begin{bmatrix}
L_{UU} & L_{Ui}\\
L_{iU} & L_{ii}
\end{bmatrix}
=
\det L_{UU}
\det\!\left(
L_{ii}-L_{iU}L_{UU}^{-1}L_{Ui}
\right).
\]
The Schur complement is scalar, yielding
Eq.~\ref{eq:app_dpp_schur}. Taking logarithms and subtracting
$\log\det L_{UU}$ proves Eq.~\ref{eq:app_dpp_marginal}.
\end{proof}

Because $L\succeq0$, the Schur complement is nonnegative whenever the inverse
exists. If the enlarged principal block is positive definite, it is strictly
positive. The expression therefore measures the additional volume contributed
by candidate $i$ after accounting for the directions already represented by
$U$. A candidate nearly contained in the span represented by $U$ receives a
small marginal gain even when its individual quality is high.

\paragraph{Empty and degenerate subsets.}
For $U=\varnothing$, we use the standard convention
$\det L_{\varnothing,\varnothing}=1$. The first marginal is therefore
\begin{equation}
\Delta_i^{\mathrm{DPP}}(\varnothing)=\log L_{ii},
\label{eq:app_dpp_empty}
\end{equation}
whenever $L_{ii}>0$.

Eq.~\ref{eq:app_dpp_marginal} is written in inverse form only when
$L_{UU}$ is nonsingular. If a kernel contains degenerate directions, the same
analysis can be applied to the strictly positive definite kernel
\begin{equation}
L^{(\varepsilon)}=L+\varepsilon I,
\qquad \varepsilon>0,
\label{eq:app_dpp_regularized}
\end{equation}
for which every principal block is invertible. The determinant and
Schur-complement identities then hold exactly with $L^{(\varepsilon)}$ in
place of $L$. Alternatively, when the current principal determinant is
positive but adding $i$ makes the enlarged determinant zero, its unregularized
log-determinant marginal is $-\infty$, reflecting that the candidate adds no
new determinantal volume. None of the arguments below depends on a particular
choice of numerical regularization.

\paragraph{Proof of the co-selection identity.}
For $i,j\notin S$ in different frames, adding $i$ changes only the
within-frame term for its own frame and adding $j$ changes only the other
frame's term. Their joint change is therefore the sum of these two changes.
Existing edges to $S$ likewise contribute the sum of the singleton edge
rewards. The only remaining edge is $\{i,j\}$, which contributes
$\gamma w_{ij}$ exactly once when both are selected. This proves
main Eq.~\ref{eq:retention_complementarity}, including the zero-reward case.
The identity is conditional on the supplied graph and kernels; it does not
certify that the graph links answer-relevant observations.
For comparison, fixed per-token rewards have zero cross-frame mixed
difference, because each token contributes independently of its counterpart.
The identity therefore distinguishes joint retention from additive scoring
while keeping the source features, per-frame kernels, and quotas fixed.

\paragraph{Adjacent-frame complementarity.}
The DPP term discourages redundant directions within a frame, but visual
similarity across neighboring frames does not necessarily imply redundancy.
To formalize the complementary pair term, let $E$ be the sparse
correspondence graph and let $w_{ij}\geq0$ be the reward associated with
edge $\{i,j\}\in E$. Define
\begin{equation}
G(S)
=
\sum_{\substack{\{i,j\}\in E\\ i,j\in S}}
w_{ij}.
\label{eq:app_pair_objective}
\end{equation}
For an unselected candidate $i\notin S$, its marginal contribution is
\begin{align}
G(S\cup\{i\})-G(S)
&=
\sum_{j\in S:\{i,j\}\in E}w_{ij}
\nonumber\\
&=
\Delta_i^{\mathrm{pair}}(S),
\label{eq:app_pair_marginal}
\end{align}
which recovers the pair reward used in the main text.

Because all $w_{ij}$ are nonnegative, this marginal exhibits
complementarity. If $S\subseteq S'$ and $i\notin S'$, then
\begin{equation}
\Delta_i^{\mathrm{pair}}(S)
\leq
\Delta_i^{\mathrm{pair}}(S').
\label{eq:app_pair_increasing}
\end{equation}
Thus selecting one endpoint can increase the value of selecting its
corresponding endpoint later.

\begin{proof}
The set of neighbors of $i$ contained in $S$ is a subset of those contained
in $S'$. Therefore
\[
\sum_{j\in S:\{i,j\}\in E}w_{ij}
\leq
\sum_{j\in S':\{i,j\}\in E}w_{ij},
\]
because every added term is nonnegative.
\end{proof}

Eq.~\ref{eq:app_pair_increasing} is the opposite direction from the
diminishing-returns condition of a submodular set function. In particular,
$G$ is supermodular. This also shows why adding the pair reward to a diversity
criterion does not, in general, preserve submodularity.

A two-candidate example makes the point explicit. Let $i$ and $j$ be joined
by a single edge with $w_{ij}>0$, and take a DPP kernel $L=I$, for which the
log-determinant contribution is zero for every selected subset. Then
\[
G(\varnothing)=G(\{i\})=G(\{j\})=0,
\qquad
G(\{i,j\})=w_{ij}.
\]
The marginal value of adding $i$ is therefore zero when starting from
$\varnothing$ but $w_{ij}$ after $j$ has been selected:
\[
G(\{i\})-G(\varnothing)
<
G(\{i,j\})-G(\{j\}).
\]
Hence the combined selection score need not satisfy a standard submodular
diminishing-returns property, and we make no corresponding approximation
guarantee for the greedy procedure.

\paragraph{Quota-constrained feasible moves.}
Let $I_t$ denote the candidates belonging to frame $t$, and let the prescribed
integer quotas satisfy
\begin{equation}
0\leq b_t\leq N_t,
\qquad
\sum_{t=1}^{T}b_t=B.
\label{eq:app_quota_constraints}
\end{equation}
At a current subset $S$, a singleton or pair move $U$ is feasible only if

\begin{equation}
\begin{gathered}
U\cap S=\varnothing,\\
|S\cup U|\leq B,\\
|(S\cup U)\cap I_t|\leq b_t
\quad\text{for every }t.
\end{gathered}
\label{eq:app_feasible_move}
\end{equation}

The pair moves allow the selector to retain correspondence endpoints jointly,
while singleton moves ensure that the quota constraints can always be
completed.

\paragraph{Termination at exactly $B$ anchors.}
Assume Eq.~\ref{eq:app_quota_constraints} and that
$\operatorname{FeasibleMoves}$ includes every unselected singleton satisfying
Eq.~\ref{eq:app_feasible_move}. Starting from $S=\varnothing$, the
loop in Algorithm~\ref{alg:image_first} cannot terminate or become stuck with
$|S|<B$ and reaches exactly $B$ distinct selected indices.

\begin{proof}
Suppose $|S|<B$. Since
\[
\sum_{t=1}^{T}b_t=B,
\]
it cannot be the case that $|S\cap I_t|=b_t$ for every frame; otherwise
summing over frames would give $|S|=B$. Hence there exists at least one frame
$t$ such that
\[
|S\cap I_t|<b_t.
\]
Because $b_t\leq N_t=|I_t|$, this frame contains at least one unselected
candidate $i\in I_t\setminus S$. The singleton move $\{i\}$ respects the
frame quota. It also respects the global budget because $|S|<B$.
Therefore at least one feasible move exists whenever the budget is not yet
filled.

Each accepted move contains only previously unselected candidates, so
$|S|$ strictly increases. Eq.~\ref{eq:app_feasible_move} prevents a
move from exceeding the global budget or any frame quota. In particular,
when only one global slot remains, no two-element pair is feasible and a
singleton remains available by the preceding argument. Since the budget is
finite, the procedure therefore reaches $|S|=B$ after finitely many moves.
All selected indices are distinct by construction.
\end{proof}

This result establishes feasibility only. The move chosen by
$\operatorname{GreedyMove}$ depends on the configured relevance, diversity,
and correspondence scores, and the complementarity term above prevents a
general claim of global optimality. The role of the procedure is instead to
construct exactly $B$ source-aligned anchors while allowing the value of a
candidate to depend on what has already been retained.\label{end:selection}

\section{Proof of the Conditional Prediction-Error Decomposition}
\label{app:information}

Fix a distribution over $(\mathcal V,q)$. Let $p_F$ be the frozen Full Image first-answer-token distribution and
$\mathcal O=(q,S,\mathbf Z_S,\mathbf P_S,\mathbf D_S)$ the selected record.
The adaptive mask $S$ is part of the observation. Define
$\bar p_{\mathcal O}=\mathbb E[p_F\mid\mathcal O]$.

\begin{IFproposition}[Conditional prediction error]
\label{prop:information}
For a finite vocabulary and finite expected KL divergence, any
$\mathcal O$-measurable predictor satisfies
\begin{equation}
\begin{aligned}
\mathbb E D_{\mathrm{KL}}(p_F\|p_\theta)
 &=\mathcal I(\mathcal O)+\mathcal A_\theta(\mathcal O),\\
\mathcal I(\mathcal O)
 &=\mathbb E D_{\mathrm{KL}}(p_F\|\bar p_{\mathcal O}),\\
\mathcal A_\theta(\mathcal O)
 &=\mathbb E D_{\mathrm{KL}}(\bar p_{\mathcal O}\|p_\theta).
\end{aligned}
\label{eq:kl_decomposition}
\end{equation}
\end{IFproposition}

Here $\mathcal I$ is the minimum achievable teacher-prediction divergence;
$\mathcal A_\theta$ measures the predictor's use of the record.
Reading context $\mathcal C$ enlarges the record to
$\mathcal O^+=(\mathcal O,\mathcal C)$ and gives
$\mathcal I(\mathcal O^+)\leq\mathcal I(\mathcal O)$.
This motivates context access without an accuracy guarantee.

We prove Proposition~\ref{prop:information} and the monotonicity statement
for additional context. Let $\mathcal Y$ denote the finite output vocabulary.
For $y\in\mathcal Y$, write
\[
\begin{gathered}
P_y=p_F(y),\\
Q_y=p_\theta(y\mid\mathcal O),\\
\bar P_y=\bar p_{\mathcal O}(y)
       =\mathbb E[P_y\mid\mathcal O].
\end{gathered}
\]
Because conditional expectation is taken coordinatewise,
$\bar p_{\mathcal O}$ is itself a probability distribution:
$\bar P_y\geq0$ and
$\sum_{y\in\mathcal Y}\bar P_y=1$ almost surely.

\begin{proof}[Proof of Proposition~\ref{prop:information}]
Using the standard convention for zero-probability terms, condition on
$\mathcal O$. For every coordinate with $\bar P_y=0$, nonnegativity of
$P_y$ and
$\mathbb E[P_y\mid\mathcal O]=0$
imply that $P_y=0$ almost surely conditional on $\mathcal O$, so that
coordinate contributes zero. On the remaining coordinates,
\begin{align}
&\mathbb E\!\left[
D_{\mathrm{KL}}(p_F\|p_\theta)
-
D_{\mathrm{KL}}(p_F\|\bar p_{\mathcal O})
\,\middle|\,\mathcal O
\right]
\nonumber\\
&\qquad=
\mathbb E\!\left[
\sum_{y\in\mathcal Y}
P_y\log\frac{\bar P_y}{Q_y}
\,\middle|\,\mathcal O
\right]
\nonumber\\
&\qquad=
\sum_{y\in\mathcal Y}
\mathbb E[P_y\mid\mathcal O]
\log\frac{\bar P_y}{Q_y}
\nonumber\\
&\qquad=
\sum_{y\in\mathcal Y}
\bar P_y\log\frac{\bar P_y}{Q_y}
=
D_{\mathrm{KL}}(\bar p_{\mathcal O}\|p_\theta).
\label{eq:conditional-kl-proof}
\end{align}
Taking expectation over $\mathcal O$ gives
\[
\mathbb E\,D_{\mathrm{KL}}(p_F\|p_\theta)
=
\mathbb E\,D_{\mathrm{KL}}
(p_F\|\bar p_{\mathcal O})
+
\mathbb E\,D_{\mathrm{KL}}
(\bar p_{\mathcal O}\|p_\theta),
\]
which is Eq.~\ref{eq:kl_decomposition}.

Since KL divergence is nonnegative,
\[
\mathbb E\,D_{\mathrm{KL}}(p_F\|p_\theta)
\geq
\mathbb E\,D_{\mathrm{KL}}
(p_F\|\bar p_{\mathcal O})
=
\mathcal I(\mathcal O).
\]
Equality is attained by the predictor
$p_\theta(\cdot\mid\mathcal O)=\bar p_{\mathcal O}$.
Hence
\[
\mathcal I(\mathcal O)
=
\inf_{p_\theta\ \mathcal O\text{-measurable}}
\mathbb E\,D_{\mathrm{KL}}(p_F\|p_\theta).
\]
\end{proof}

\paragraph{Additional context can only reduce the information floor.}
Let $\mathcal O^+=(\mathcal O,\mathcal C)$ be an enriched record and define
\[
\bar p_{\mathcal O^+}
=
\mathbb E[p_F\mid\mathcal O^+].
\]
Because $\mathcal O$ is contained in $\mathcal O^+$,
$\bar p_{\mathcal O}$ is also measurable with respect to
$\mathcal O^+$. Applying Proposition~\ref{prop:information} with
$\mathcal O^+$ as the observed record and
$\bar p_{\mathcal O}$ as the predictor gives
\begin{align}
\mathcal I(\mathcal O)
&=
\mathbb E\,D_{\mathrm{KL}}
(p_F\|\bar p_{\mathcal O})
\nonumber\\
&=
\mathbb E\,D_{\mathrm{KL}}
(p_F\|\bar p_{\mathcal O^+})
+
\mathbb E\,D_{\mathrm{KL}}
(\bar p_{\mathcal O^+}\|\bar p_{\mathcal O})
\nonumber\\
&=
\mathcal I(\mathcal O^+)
+
\mathbb E\,D_{\mathrm{KL}}
(\bar p_{\mathcal O^+}\|\bar p_{\mathcal O}).
\label{eq:context_information_identity}
\end{align}
Therefore
\[
\mathcal I(\mathcal O^+)\leq\mathcal I(\mathcal O),
\]
with equality if and only if
$\bar p_{\mathcal O^+}=\bar p_{\mathcal O}$ almost surely on the relevant
support. Thus additional context lowers the irreducible teacher-prediction
gap exactly to the extent that it changes the conditional teacher
distribution.\label{end:information}

\section{Proof of the Selection--Interaction Identity}
\label{app:cut}

For a fixed diagnostic operator $F_A(\mathbf Z)=(I+A)\mathbf Z$, write
$A_{RQ}=P_RAP_Q^\top$ and let $\bar S$ be the unselected indices.

\begin{IFproposition}[Omitted interaction]
\label{prop:cut}
For fixed $A$ and $S$,
\begin{equation}
\begin{aligned}
&P_S(I+A)\mathbf Z-(I_B+A_{SS})\mathbf Z_S\\
&\hspace{4em}=A_{S\bar S}\mathbf Z_{\bar S}.
\end{aligned}
\label{eq:cut_identity}
\end{equation}
The two orders agree for every $\mathbf Z$ iff $A_{S\bar S}=0$.
\end{IFproposition}

Both orders output $B$ rows. Reading unselected neighbors can restore
interaction terms excluded by selected-only refinement. This is a
fixed-operator identity, not a linearity assumption for the backbone or
an answer-accuracy bound. The proof and worst-case residual follow below.

Main Fig.~\ref{fig:interaction_order} illustrates these fixed-operator support patterns.

We prove Proposition~\ref{prop:cut} and the worst-case residual in
Eq.~\ref{eq:cut_minimax}. For any index set $R$, let $P_R$ be its
row-selection matrix, and write
\[
\mathbf Z_R=P_R\mathbf Z,
\qquad
A_{RQ}=P_RAP_Q^\top.
\]

\begin{proof}[Proof of Proposition~\ref{prop:cut}]
Because $S$ and $\bar S$ partition the $M$ candidate indices,
\[
P_S^\top P_S+P_{\bar S}^\top P_{\bar S}=I_M,
\]
and therefore
\[
\mathbf Z
=
P_S^\top\mathbf Z_S
+
P_{\bar S}^\top\mathbf Z_{\bar S}.
\]
Expanding the interact-then-select operator gives
\begin{align}
P_S(I+A)\mathbf Z
&=
\mathbf Z_S+P_SA\mathbf Z
\nonumber\\
&=
\mathbf Z_S
+
P_SAP_S^\top\mathbf Z_S
+
P_SAP_{\bar S}^\top\mathbf Z_{\bar S}
\nonumber\\
&=
(I_B+A_{SS})\mathbf Z_S
+
A_{S\bar S}\mathbf Z_{\bar S}.
\end{align}
Subtracting the select-then-interact term yields
\[
P_S(I+A)\mathbf Z
-
(I_B+A_{SS})\mathbf Z_S
=
A_{S\bar S}\mathbf Z_{\bar S},
\]
which proves Eq.~\ref{eq:cut_identity}.

If $A_{S\bar S}=0$, the two orders clearly agree for every
$\mathbf Z$. Conversely, suppose they agree for every $\mathbf Z$. Then
\[
A_{S\bar S}\mathbf Z_{\bar S}=0
\]
for every possible $\mathbf Z_{\bar S}$. Choosing a matrix whose one
feature column is an arbitrary vector
$v\in\mathbb R^{|\bar S|}$ and whose remaining columns are zero gives
$A_{S\bar S}v=0$ for every $v$. Hence
$A_{S\bar S}=0$.
\end{proof}

\paragraph{Partial context retrieval.}
Let $C\subseteq\bar S$ denote the retrieved context indices and let
\[
U=\bar S\setminus C
\]
be the remaining unobserved indices. The full selected output after applying
the fixed interaction operator can be written as
\begin{equation}
P_S(I+A)\mathbf Z
=
(I_B+A_{SS})\mathbf Z_S
+
A_{SC}\mathbf Z_C
+
A_{SU}\mathbf Z_U.
\label{eq:full-interaction-partition}
\end{equation}
If the selected representation has access to $S$ and $C$ but not $U$, then
the term attributable to the unavailable indices is exactly
\[
E_U=A_{SU}\mathbf Z_U.
\]

We now compute its worst-case magnitude under the constraint
$\|\mathbf Z_U\|_F\leq R$. By submultiplicativity of the Frobenius norm,
\[
\|A_{SU}\mathbf Z_U\|_F
\leq
\|A_{SU}\|_2\|\mathbf Z_U\|_F
\leq
R\|A_{SU}\|_2.
\]
Thus
\[
\sup_{\|\mathbf Z_U\|_F\leq R}
\|A_{SU}\mathbf Z_U\|_F
\leq
R\|A_{SU}\|_2.
\]

To show that the bound is attained, let
$v\in\mathbb R^{|U|}$ be a unit right singular vector of $A_{SU}$
associated with its largest singular value
$\sigma_{\max}(A_{SU})=\|A_{SU}\|_2$, and let
$e\in\mathbb R^D$ be any unit vector. Set
\[
\mathbf Z_U=Rve^\top.
\]
Then
\[
\|\mathbf Z_U\|_F
=
R\|v\|_2\|e\|_2
=
R,
\]
while
\begin{align}
\|A_{SU}\mathbf Z_U\|_F
&=
R\|A_{SU}ve^\top\|_F \\
&=
R\|A_{SU}v\|_2\|e\|_2 \\
&=
R\|A_{SU}\|_2.
\end{align}
Hence
\begin{equation}
\sup_{\|\mathbf Z_U\|_F\leq R}
\|A_{SU}\mathbf Z_U\|_F
=
R\|A_{SU}\|_2.
\label{eq:cut_minimax}
\end{equation}

This result is conditional on the fixed operator $A$, the fixed selection
set $S$, and the retrieved context set $C$. It characterizes the largest
interaction contribution hidden behind the selection boundary under a
norm-bounded unobserved feature matrix; it is not a bound on downstream
answer accuracy or on an adaptive nonlinear video encoder.\label{end:cut}

\section{Temporal Refinement and Learning Objective}
\label{app:refinement}

This supplementary section gives additional details for the temporal refinement module in
main Section~\ref{sec:temporal_refinement}. The purpose of refinement is to let
each selected source anchor read temporally relevant evidence without
increasing the number of visual tokens forwarded to the language model.
Throughout refinement, the selected index set $S$ and its source metadata
remain fixed.

\paragraph{Context retrieval.}
For each selected anchor $i\in S$, let $C(i)$ denote a set of retrieved
tokens from neighboring frames. The context set is not restricted to the
selected anchors:
\begin{equation}
  C(i)\subseteq \{1,\ldots,M\}\setminus\{i\},
  \qquad
  C(i)\not\subseteq S
  \ \text{in general}.
  \label{eq:app_context_set}
\end{equation}
Thus selection determines which tokens occupy the output sequence, whereas
context retrieval determines which additional tokens may be read while
constructing those outputs. Tokens used only as context do not become new
language-model input tokens.

For a selected anchor $i$ and context token $j\in C(i)$, the refinement
module uses their visual features together with relative temporal and spatial
information to produce normalized attention weights
$\alpha_{ij}$. We require
\begin{equation}
  \alpha_{ij}\geq 0,
  \qquad
  \sum_{j\in C(i)}\alpha_{ij}=1.
  \label{eq:app_attention_normalization}
\end{equation}
The aggregated context for anchor $i$ is
\begin{equation}
  \mathbf c_i
  =
  \sum_{j\in C(i)}
  \alpha_{ij}W_V\mathbf z_j,
  \label{eq:app_context_aggregation}
\end{equation}
where $W_V$ is the value projection used by the bottleneck attention module.

\paragraph{Residual refinement.}
The selected anchor is updated through the gated residual branch
$G_\theta$:
\begin{equation}
  \widehat{\mathbf z}_i
  =
  \mathbf z_i
  +
  G_\theta(\mathbf z_i,\mathbf c_i),
  \qquad i\in S.
  \label{eq:app_refinement}
\end{equation}
This is the same update as main Eq.~\ref{eq:refinement}, with
Eq.~\ref{eq:app_context_aggregation} making the retrieved context
explicit.

The residual branch modifies the content associated with an existing
anchor rather than creating additional output anchors. The zero-expansion
property below follows from producing one update per selected anchor; it
does not require a uniform bound on the residual norm.

\paragraph{Zero-expansion interface.}
Let
\[
\widehat{\mathbf Z}_S
=
\operatorname{Concat}
\left(
\widehat{\mathbf z}_i:i\in S
\right).
\]
Since each of the $B$ selected anchors produces exactly one refined output,
\begin{equation}
  \widehat{\mathbf Z}_S
  \in
  \mathbb R^{B\times D}.
  \label{eq:app_zero_expansion}
\end{equation}
The source positions and intermediate feature streams remain
\begin{equation}
  \widehat{\mathbf P}_S=\mathbf P_S,
  \qquad
  \widehat{\mathbf D}_S=\mathbf D_S.
  \label{eq:app_source_metadata}
\end{equation}
Hence refinement is a $B\!\to\!B$ operation at the visual output interface,
even though its computation may read tokens outside $S$.

This distinction is important. The dense image pool can serve as temporary
context memory without becoming the sequence supplied to the language model.
Zero expansion therefore refers to output-token count, not to the amount of
computation or memory used while producing the refined anchors.

\paragraph{Relation to the interaction boundary.}
Proposition~\ref{prop:cut} shows that selecting before interaction omits
cross-boundary terms of the form
$A_{S\bar S}\mathbf Z_{\bar S}$ in the fixed linear diagnostic model.
Retrieving context $C\subseteq\bar S$ restores access to the corresponding
component
\[
A_{SC}\mathbf Z_C,
\]
while the remaining unavailable component is associated with
$U=\bar S\setminus C$.

The learned refiner is not assumed to implement this linear operator exactly.
The proposition instead motivates its access pattern: selected anchors should
be allowed to read relevant unselected context when useful. In this sense,
selection fixes the output support, while refinement relaxes the interaction
boundary induced by that support.

\paragraph{Relation to the information decomposition.}
The same distinction can be expressed through
Proposition~\ref{prop:information}. Before refinement, the selected record is
\[
\mathcal O
=
(q,S,\mathbf Z_S,\mathbf P_S,\mathbf D_S).
\]
After retrieving context $\mathcal C$, the available record becomes
\[
\mathcal O^+
=
(\mathcal O,\mathcal C).
\]
As shown in Eq.~\ref{eq:context_information_identity},
\begin{equation}
  \mathcal I(\mathcal O)
  =
  \mathcal I(\mathcal O^+)
  +
  \mathbb E
  D_{\mathrm{KL}}
  \left(
    \bar p_{\mathcal O^+}
    \|
    \bar p_{\mathcal O}
  \right).
  \label{eq:app_refinement_information}
\end{equation}
Thus context retrieval can reduce the information floor by making additional
evidence observable. The residual branch then determines how effectively the
available record is converted into a prediction. More expressive refinement
can improve the latter operation, but it cannot recover evidence that neither
selection nor context retrieval exposes.

\textbf{Learning objective.}
Only the temporal refiner parameters $\theta$ are optimized. The image
encoder, language model, and selector remain frozen. Let
$p_\theta(\cdot\mid\mathcal V,q)$ denote the selected-and-refined pathway's
first-answer-token distribution, and let
$p_F(\cdot\mid\mathcal V,q)$ denote the frozen Full Image
first-answer-token distribution.

Training combines answer supervision and teacher distillation:
\begin{equation}
  \mathcal L(\theta)
  =
  \mathcal L_{\mathrm{ans}}(\theta)
  +
  \beta\,\mathcal L_{\mathrm{KD}}(\theta),
  \label{eq:app_training_objective}
\end{equation}
where $\beta$ controls the relative contribution of the distillation term.
At temperature $\tau$, write
$p_F^{(\tau)}$ and $p_\theta^{(\tau)}$ for the corresponding softened
distributions. The distillation term is
\begin{equation}
  \mathcal L_{\mathrm{KD}}(\theta)
  =
  D_{\mathrm{KL}}
  \left(
    p_F^{(\tau)}
    \|
    p_\theta^{(\tau)}
  \right).
  \label{eq:app_kd_objective}
\end{equation}
The Full Image pathway is used as a teacher only during training; it is not
required at inference.

\paragraph{Initialization and the training-free limit.}
The residual branch is initialized so that its initial output is zero:
\begin{equation}
  G_{\theta_0}(\mathbf z_i,\mathbf c_i)=0.
  \label{eq:app_zero_initialization}
\end{equation}
Consequently,
\begin{equation}
  \widehat{\mathbf z}_i
  =
  \mathbf z_i
  \qquad
  \text{at initialization}.
  \label{eq:app_initial_identity}
\end{equation}
Optimization therefore begins from the unrefined selected-token pathway
rather than from an arbitrary modification of the retained image features.
\begin{algorithm}[t]
  \caption{\textbf{Image-first encoding with a fixed output budget.} Joint selection gathers frame features, positions, and intermediate streams at the same source indices. Optional refinement preserves the $B$-token interface to the frozen language model.}
  \label{alg:image_first}
  \centering
  \begin{tabular}{@{}r@{\hspace{0.8em}}p{\dimexpr\linewidth-2.1em\relax}@{}}
    \toprule
    \multicolumn{2}{@{}p{\linewidth}@{}}{
    \textbf{Input:} video $\mathcal V$, question $q$, budget $B$,
frozen image encoder, frame-budget allocator, optional refiner $G_\theta$.}\\
    \midrule
    1 &
    \IFrep{Encode every frame independently}; collect
    $(\mathbf Z,\mathbf P,\mathbf D)$.\\
    2 &
    \IFalloc{Compute quotas} $b_t$ as in Supplementary Sec.~\ref{app:selector_quota}, with
$0\leq b_t\leq N_t$ and $\sum_t b_t=B$.\\
    3 &
    Build question-weighted DPP kernels $\{L_t\}$ and pair graph $(E,w)$.\\
    4 &
    $S\gets\emptyset$.\\
    5 &
    \textbf{while} $|S|<B$ \textbf{do}\\
    6 &
    \hspace*{1em}
    $\mathcal U\gets
    \operatorname{FeasibleMoves}(S,\{b_t\},E)$.\\
    7 &
    \hspace*{1em}
    $U\gets
    \operatorname{GreedyMove}(\mathcal U;S,\{L_t\},w)$.\\
    8 &
    \hspace*{1em}
    $S\gets S\cup U$.\\
    9 &
    \textbf{end while}\\
    10 &
    Sort $S$ by source time and spatial position.\\
    11 &
    $(\widehat{\mathbf Z}_S,\mathbf P_S,\mathbf D_S)
    \gets
    (\mathbf Z[S],\mathbf P[S],\mathbf D[S])$.\\
    12 &
    \textbf{if} temporal refinement is enabled \textbf{then}\\
    13 &
    \hspace*{1em}
    \textbf{for each} selected anchor $i\in S$ \textbf{do}\\
    14 &
    \hspace*{2em}
    \IFinter{Retrieve context} $C(i)$ from the unchanged dense pool
    $\mathbf Z$.\\
    15 &
    \hspace*{2em}
    \IFinter{Write one residual update} using
    main Eq.~\ref{eq:refinement}.\\
    16 &
    \hspace*{1em}
    \textbf{end for}\\
    17 &
    \textbf{end if}\\
    18 &
    \textbf{return}
    $(\widehat{\mathbf Z}_S,\mathbf P_S,\mathbf D_S)$:
    exactly $B$ source-aligned tokens.\\
    \bottomrule
    
  \end{tabular}
  
\end{algorithm}

Disabling the residual branch gives the exact training-free limit
\begin{equation}
  \widehat{\mathbf Z}_S=\mathbf Z_S.
  \label{eq:app_no_refinement}
\end{equation}
The coordinated selector and the fixed-budget source interface remain
unchanged. The learned component is therefore confined to temporal
refinement: representation and selection do not require additional training.

\paragraph{Scope of the refinement claim.}
The refinement module does not claim to reconstruct the discarded dense image
sequence or reproduce the native video encoder. Its role is narrower: it
allows a fixed set of source-aligned anchors to incorporate selected temporal
context while preserving the output-token budget and source interface.
Accordingly, the $B\!\to\!B$ property is a statement about the representation
passed to the language model, not a claim that context retrieval and
refinement incur zero computational cost.\label{end:refinement}

\section{Selector Specification and Hyperparameters}
\label{app:selector_spec}

This section specifies the question-relevance score, the per-frame DPP kernel, frame-budget allocation, the adjacent-frame correspondence graph, the greedy move scores, and the context size used by the temporal refiner. The allocation parameters are fixed at inference time.

\subsection{Question-Conditioned Relevance}
\label{app:selector_relevance}

For each question, we encode the complete question together with up to eight question/option fragments. Let $Q_r$ denote the input-token indices of text fragment $r$. We form one vector per fragment by averaging the corresponding language-model input embeddings,
\begin{equation}
q_r = \frac{1}{|Q_r|}\sum_{k\in Q_r} e_k .
\label{eq:selector_text_embedding}
\end{equation}
For a visual candidate $z_{ti}$ from frame $t$, both the visual feature and text vector are $\ell_2$-normalized,
\begin{equation}
\bar z_{ti}=\frac{z_{ti}}{\|z_{ti}\|_2},
\qquad
\bar q_r=\frac{q_r}{\|q_r\|_2}.
\label{eq:selector_normalization}
\end{equation}
The question relevance used by the implementation is the maximum cosine similarity across the encoded text fragments,
\begin{equation}
m_{ti}(q)=\max_r \bar z_{ti}^{\top}\bar q_r.
\label{eq:selector_question_relevance}
\end{equation}
Because no additional clipping is applied, $m_{ti}(q)\in[-1,1]$. This is the range used by the implementation.

\subsection{Token Quality and Per-Frame DPP Kernel}
\label{app:selector_dpp}

The unnormalized quality of candidate $i$ in frame $t$ is
\begin{equation}
\widetilde a_{ti}
=
\max\!\left\{
10^{-6},
\|z_{ti}\|_2\left[1+\eta m_{ti}(q)\right]
\right\},
\qquad \eta=\alpha_q.
\label{eq:selector_quality_raw}
\end{equation}
Here $\alpha_q$ is the configured question weight in Supplementary Sec.~\ref{app:selector_quota} (0.8 in the listed preset). We normalize the qualities by their mean within each frame,
\begin{equation}
a_{ti}
=
\frac{\widetilde a_{ti}}
{\frac{1}{N_t}\sum_j \widetilde a_{tj}}.
\label{eq:selector_quality_normalized}
\end{equation}
The DPP kernel for frame $t$ is then
\begin{equation}
L_t(i,j)
=
a_{ti}a_{tj}\,\bar z_{ti}^{\top}\bar z_{tj}
+
\epsilon_t\mathbf{1}[i=j],
\label{eq:selector_dpp_kernel}
\end{equation}
with
\begin{equation}
\epsilon_t
=
10^{-3}\frac{1}{N_t}\sum_i a_{ti}^{2}.
\label{eq:selector_dpp_jitter}
\end{equation}
The diagonal term is a numerical stabilizer. Selection within each frame is therefore driven jointly by candidate quality and diversity in the normalized visual-feature space.

\subsection{Training-Free Frame-Budget Allocation}
\label{app:selector_quota}

We compute frame quotas once and hold them fixed during coordinated token
selection. The score definitions below make the spatial and text-fragment
aggregation explicit. For the shared spatial grid in our evaluations, the
visual tensor has shape $T\times N\times D$; each question fragment has one
$D$-dimensional vector as defined in Eq.~\eqref{eq:selector_text_embedding}.
Let $n_{ti}=\|z_{ti}\|_2$, $r_{ti}=\max\{0,m_{ti}(q)\}$,
$k=\min\{8,N\}$, and $K_t$ index the $k$ largest $r_{ti}$ within frame $t$.
The question, saliency, and diversity statistics are
\begin{equation}
\begin{aligned}
s_t^{\mathrm{qa}}&=\frac{1}{k}\sum_{i\in K_t}r_{ti},
&s_t^{\mathrm{sal}}&=\frac{1}{N}\sum_i n_{ti},\\
p_{ti}&=\frac{n_{ti}}{\sum_j n_{tj}+\epsilon},
&s_t^{\mathrm{div}}&=-\sum_i p_{ti}\log(p_{ti}+\epsilon),
\end{aligned}
\label{eq:quota_spatial_statistics}
\end{equation}
where $\epsilon=10^{-6}$. Thus question relevance first maximizes over
text fragments for each spatial candidate, then averages the top spatial
responses. The nonnegative clipping here is specific to the frame score;
the DPP quality retains the signed $m_{ti}(q)$ from
Eq.~\eqref{eq:selector_quality_raw}. Saliency and norm entropy use the
unnormalized visual features.

Motion uses normalized features at corresponding spatial positions. Let
$\mathcal N(t)=\{t-1,t+1\}\cap\{1,\ldots,T\}$ for $T\geq2$. Then
\begin{equation}
s_t^{\mathrm{mot}}=\frac{1}{N}\sum_i
\max_{v\in\mathcal N(t)}\|\bar z_{ti}-\bar z_{vi}\|_2.
\label{eq:quota_motion_statistic}
\end{equation}
The maximum is over the two temporal neighbors before spatial averaging;
an endpoint has only one neighbor. Each of the four statistics is
min--max normalized across frames using the configured normalization
rule before taking a weighted sum,
\begin{equation}
  s_t=1+\alpha_s\widetilde s_t^{\mathrm{sal}}
       +\alpha_d\widetilde s_t^{\mathrm{div}}
       +\alpha_m\widetilde s_t^{\mathrm{mot}}
       +\alpha_q\widetilde s_t^{\mathrm{qa}}.
\label{eq:quota_frame_score}
\end{equation}
Temporal non-maximum suppression selects preview frames with the configured minimum gap. Their scores are boosted, and an uncertainty factor further upweights frames with diffuse question relevance. Uncertainty is the normalized entropy of the within-frame relevance softmax (temperature 0.07), multiplied by the exponential of the negative top-two relevance margin; it is set to one when all relevances vanish.

The requested feasible budget $B$ is divided into configured base and
residual portions. Before score-based allocation, reserve a common floor
\begin{equation}
\begin{aligned}
f(B,T)&=\min\!\left\{12,\left\lfloor\frac{B}{2T}\right\rfloor\right\},\\
\ell_t&=\min\{N_t,f(B,T)\},\\
B_{\mathrm{floor}}&=\sum_t\ell_t,\qquad
B_{\mathrm{dyn}}=B-B_{\mathrm{floor}}.
\end{aligned}
\label{eq:floor_half_budget_reserve}
\end{equation}
Thus $B_{\mathrm{floor}}\leq B/2$. The reserve is charged once to the
base portion; the residual portion has no additional guaranteed floor.
For $B=1536$ with adequate capacities, $T=32,64,128$ gives floors
$12,12,6$, reserves $384,768,768$, and dynamic budgets $1152,768,768$.
The dynamic budget is distinct from the configured residual slot portion:
score-based allocation can occur within either portion.

Within each portion, the non-reserved slots are allocated proportionally
to the weighted frame scores, under capacities, by floor rounding and
largest-remainder assignment. The residual portion applies the configured
additional uncertainty weighting. Ties are resolved by the configured
score order and then frame index. Base slots use question-aware DPP
anchors and norm-weighted merging; residual slots supply additional
evidence. These provisional features are used to compute source counts;
the coordinated selector subsequently gathers unmodified source tokens.
Table~\ref{tab:allocator_hparams} lists the frame-budget allocation preset.

Let $B_{\mathrm{prov}}\leq B$ be the number of provisional slots, which can underfill before quota correction. Each slot has a representative source index $r_\ell$; for a merged slot, this is the contributing source token with the highest cosine similarity to its feature. Each sampled frame contributes the same number $N$ of candidates. Using one-based frame indices, the representative's source frame is
\begin{equation}
\operatorname{frame}(r_\ell)
=
1+\left\lfloor\frac{r_\ell}{N}\right\rfloor,
\label{eq:selector_rep_frame}
\end{equation}
where $r_\ell$ is the zero-based global source index stored by the implementation. The representative count for frame $t$ is
\begin{equation}
c_t
=
\sum_{\ell=1}^{B_{\mathrm{prov}}}
\mathbf{1}\!\left[\operatorname{frame}(r_\ell)=t\right].
\label{eq:selector_rep_count}
\end{equation}
The initial quota is
\begin{equation}
b_t=\min\{c_t,N\}.
\label{eq:selector_initial_quota}
\end{equation}
If clipping causes $\sum_t b_t<B$, the remaining slots are assigned iteratively to frames with the largest residual capacity $N-b_t$ until
\begin{equation}
\sum_t b_t=B.
\label{eq:selector_quota_sum}
\end{equation}
Ties in this capacity correction favor the first frame index. The coordinated
selector preserves these frame quotas while optimizing which source tokens
occupy the slots; it does not pass the provisional merged features to the
language model.

\begin{table*}[t]
\caption{\textbf{Frame-budget allocation parameters.} Weights combine the signals used to allocate per-frame quotas. The floor is bounded by 12 and half the total budget; it equals 12 for 32 frames and 1,536 tokens.}
\label{tab:allocator_hparams}
\centering

\begingroup
\small
\setlength{\tabcolsep}{4.0pt}
\renewcommand{\arraystretch}{1.05}

\begin{tabularx}{\linewidth}
{@{}*{5}{>{\centering\arraybackslash}X}@{}}
\toprule
\textbf{Question weight} &
\textbf{Motion weight} &
\textbf{Saliency weight} &
\textbf{Diversity weight} &
\textbf{Frame floor} \\
\midrule
0.8 & 0.6 & 0.4 & 0.4 & 12 \\
\midrule
\textbf{Preview frames} &
\textbf{Preview boost} &
\textbf{Uncertainty weight} &
\textbf{Temporal NMS gap} &
\textbf{DPP ratio} \\
\midrule
5 & 0.75 & 0.45 & 2 & 0.5 \\
\bottomrule
\end{tabularx}

\endgroup

\end{table*}

\begin{table*}[t]
\caption{\textbf{Selection and refinement implementation constants.} Blocks specify token-quality stabilization, correspondence gates, pair rewards, and context sizes. Spatial distances are normalized; context counts do not increase the output-token budget.}
\label{tab:selector_hparams}
\centering

\begingroup
\small
\setlength{\tabcolsep}{2.0pt}
\renewcommand{\arraystretch}{1.04}

\begin{tabularx}{\linewidth}
{@{}*{7}{>{\centering\arraybackslash}X}@{}}
\toprule
\shortstack{\textbf{Relevance}\\\textbf{weight} $\eta$} &
\shortstack{\textbf{Minimum}\\\textbf{raw}\\\textbf{quality}} &
\shortstack{\textbf{Diagonal}\\\textbf{jitter}\\\textbf{coefficient}} &
\shortstack{\textbf{Residual}\\\textbf{floor}} &
\shortstack{\textbf{Cosine}\\\textbf{threshold}} &
\shortstack{\textbf{Margin}\\\textbf{threshold}} &
\shortstack{\textbf{Maximum}\\\textbf{spatial}\\\textbf{distance}} \\
\midrule
$\alpha_q$ & $10^{-6}$ & $10^{-3}$ & $10^{-14}$ & 0.65 & 0.005 & 0.5 \\
\midrule
\shortstack{\textbf{Scene}\\\textbf{similarity}\\\textbf{threshold}} &
\shortstack{\textbf{Confidence}\\\textbf{margin scale}} &
\shortstack{\textbf{Spatial}\\\textbf{change}\\\textbf{scale}} &
\shortstack{\textbf{Feature}\\\textbf{change}\\\textbf{multiplier}} &
\shortstack{\textbf{Pair}\\\textbf{reward}\\\textbf{weight} $\gamma$} &
\shortstack{\textbf{Context}\\\textbf{size} $K$} &
\shortstack{\textbf{Neighbors}\\\textbf{per}\\\textbf{context frame}} \\
\midrule
0.65 & 0.05 & 0.25 & 2.0 & 2.0 & 8 & 4 \\
\bottomrule
\end{tabularx}

\endgroup
\end{table*}
\subsection{Adjacent-Frame Correspondence Graph}
\label{app:selector_correspondence}

Correspondence edges are constructed only between adjacent frames. For candidates $i$ in frame $t$ and $j$ in frame $t+1$, define the cosine similarity
\begin{equation}
s_{ij}=\bar z_{ti}^{\top}\bar z_{t+1,j}.
\label{eq:selector_pair_similarity}
\end{equation}
A candidate pair is considered only when $j$ is the top-1 match of $i$ and $i$ is also the top-1 reverse match of $j$. Let $\delta_{i\rightarrow j}$ and $\delta_{j\rightarrow i}$ denote the corresponding top-1 minus top-2 similarity margins in the forward and reverse directions. We use
\begin{equation}
\Delta_{ij}=\min\{\delta_{i\rightarrow j},\delta_{j\rightarrow i}\}.
\label{eq:selector_pair_margin}
\end{equation}
Let $p_i,p_j\in[0,1]^2$ be normalized patch coordinates. Their normalized spatial distance is
\begin{equation}
d_{ij}=\frac{\|p_i-p_j\|_2}{\sqrt{2}}.
\label{eq:selector_spatial_distance}
\end{equation}
Adjacent-frame scene consistency is the directed feature-coverage mean
\begin{equation}
s_t^{\mathrm{scene}}
=\frac{1}{N_t}\sum_{i=1}^{N_t}
\max_{1\leq j\leq N_{t+1}}
\bar z_{ti}^{\top}\bar z_{t+1,j}.
\label{eq:scene_directed_coverage}
\end{equation}
Several source tokens may share a best match, so this is a directed
nearest-neighbor reduction, not a one-to-one or spatially constrained
assignment. The reverse statistic need not equal the forward statistic.
The row maxima reuse the adjacent-frame similarity matrix already
computed for matching. The 0.65 gate suppresses frame pairs with low
feature coverage; it does not guarantee shot-boundary detection.
A mutual match is retained as an edge only if
\begin{equation}
\begin{gathered}
s_{ij}\ge0.65,
\qquad
\Delta_{ij}\ge0.005,
\\
d_{ij}\le0.5,
\qquad
s_t^{\mathrm{scene}}\ge0.65.
\end{gathered}
\label{eq:selector_edge_filter}
\end{equation}
For retained edges, the matching confidence is
\begin{equation}
c_{ij}
=
\operatorname{clip}\!\left(\frac{s_{ij}-0.65}{0.35},0,1\right)
\operatorname{clip}\!\left(\frac{\Delta_{ij}}{0.05},0,1\right),
\label{eq:selector_match_confidence}
\end{equation}
and the state-change score is
\begin{equation}
h_{ij}
=
\max\!\left\{
\operatorname{clip}\!\left(\frac{d_{ij}}{0.25},0,1\right),
\operatorname{clip}\!\left(2[1-s_{ij}],0,1\right)
\right\}.
\label{eq:selector_state_change}
\end{equation}
The final correspondence reward is
\begin{equation}
w_{ij}=c_{ij}h_{ij}.
\label{eq:selector_edge_weight}
\end{equation}
This weighting favors mutually reliable correspondences whose endpoints also exhibit a measurable spatial or feature change.

\subsection{Selection Objective and Greedy Solver}
\label{app:selector_objective}

Let $S_t$ denote the selected candidates from frame $t$, and let $E$ be the correspondence graph defined above. The numerical selection objective is
\begin{equation}
\begin{aligned}
J(S)
&=
\sum_{t=1}^{T}
\log\det L_t[S_t,S_t]\\
&+
\gamma
\sum_{\{i,j\}\in E}
w_{ij}\mathbf{1}[i,j\in S],
\qquad \gamma=2.0,
\end{aligned}
\label{eq:selector_exact_objective}
\end{equation}
subject to
\begin{equation}
|S_t|=b_t,
\qquad
\sum_t b_t=B.
\label{eq:selector_exact_constraints}
\end{equation}
Equation~\ref{eq:selector_exact_objective} restates the main-paper objective (Eq.~\ref{eq:selection_principle}); the solver below specifies its greedy move scores.

For greedy evaluation, let $U_t$ be the candidates already selected in frame $t$. The DPP Schur-complement residual for adding candidate $i$ is
\begin{equation}
\rho_i^{\mathrm{DPP}}
=
\begin{cases}
L_t(i,i), & U_t=\varnothing,\\[2mm]
\begin{aligned}[t]&L_t(i,i)\\&{}-L_t(i,U_t)\,L_t[U_t,U_t]^{-1}\,L_t(U_t,i)\end{aligned}, & \text{otherwise}.
\end{cases}
\label{eq:selector_dpp_residual}
\end{equation}
For numerical stability, the residual is floored at $10^{-14}$ before taking the logarithm. The singleton move score is
\begin{equation}
g_i
=
\log\max\{\rho_i^{\mathrm{DPP}},10^{-14}\}
+
\gamma
\sum_{j\in S:\{i,j\}\in E}w_{ij}.
\label{eq:selector_singleton_score}
\end{equation}
For a feasible correspondence edge $\{i,j\}\in E$, the pair score is
\begin{equation}
g_{ij}
=
\frac{g_i+g_j+\gamma w_{ij}}{2}.
\label{eq:selector_pair_score}
\end{equation}

\paragraph{Specified \texttt{GreedyMove}.}
Given the current selected set $S$, one greedy iteration performs the following steps:
\begin{enumerate}[leftmargin=*,itemsep=0pt]

    \item enumerate all quota-feasible unselected singletons and compute $g_i$ using Eq.~\ref{eq:selector_singleton_score};
    \item let $i^\star$ be the feasible singleton with the largest score;
    \item enumerate all quota-feasible correspondence edges $\{i,j\}\in E$ and compute $g_{ij}$ using Eq.~\ref{eq:selector_pair_score};
    \item let $\{u^\star,v^\star\}$ be the feasible pair with the largest pair score; if its score is strictly larger than $g_{i^\star}$, add both endpoints, otherwise add $i^\star$;
    \item repeat until all frame quotas are satisfied and $|S|=B$.
    
\end{enumerate}
The availability of feasible singleton moves guarantees completion whenever slots remain. The solver is greedy and makes no claim of globally maximizing Eq.~\ref{eq:selector_exact_objective}.

\subsection{Temporal Context Retrieval}
\label{app:selector_context}

The temporal refiner uses a fixed context size of
\begin{equation}
K=8.
\label{eq:selector_context_k}
\end{equation}
For a selected anchor $i$ in an interior frame $t$, we retrieve the four most similar tokens by cosine similarity from the previous frame and the four most similar tokens from the next frame,
\begin{equation}
C(i)
=
\operatorname{Top4}_{t-1}(i)
\cup
\operatorname{Top4}_{t+1}(i).
\label{eq:selector_context_interior}
\end{equation}
For the first frame, the two context frames are frames 2 and 3; for the last frame, they are frames $T-1$ and $T-2$. Context is retrieved from the unchanged dense image pool and may therefore include tokens that were not selected into the $B$ output anchors.

For each anchor-context pair, the refiner receives the relative-position feature
\begin{equation}
\left[
\operatorname{sign}(\Delta t)\log(1+|\Delta t|),
\Delta x,
\Delta y
\right],
\label{eq:selector_relative_position}
\end{equation}
where $\Delta x$ and $\Delta y$ are normalized spatial offsets. The retrieved context is used only to update the retained anchors; it does not
increase the number of visual tokens passed to the language model.
Table~\ref{tab:selector_hparams} summarizes the remaining selection and
context-retrieval constants.

\paragraph{Fixed-budget interface invariants.}
The selector and refiner act at different boundaries, so their resource
counts should be read separately. Let $N_t$ be the number of candidates in
frame $t$ and $M=\sum_t N_t$ the size of the image-feature pool. The frame
quotas constrain the selected subsets through
\[
  |S_t|=b_t\leq N_t,\qquad
  |S|=\sum_t|S_t|=\sum_t b_t=B.
\]
These equalities describe the output interface after allocation and
selection. Increasing the sampled frame count can increase $M$ while
leaving $B$ unchanged. It therefore changes the available evidence and
encoding work even when the number of visual tokens passed to the
language model remains fixed. The common floor and capacity correction ensure feasible per-frame
quotas; they do not remove the cost
of constructing the candidate pool.

The selection matrix applies the same index set to features, source
positions, and intermediate feature streams. Sorting selected anchors by
source time and position is consequently a joint permutation of all three
streams. This preserves the association between a selected feature and
its source metadata; independently reordering one stream would change the
interface. Before refinement, each selected row is an original candidate
vector, so a different selected set changes membership without averaging
or interpolating the source features.

Context retrieval changes the information available to each selected
anchor, but the retrieved neighbors are not appended as extra language-model
tokens. Every selected anchor produces one updated vector. Thus the
unrefined and refined interfaces both have $B$ rows, although the latter
may read candidates outside $S$. A fixed output-token budget therefore
does not imply identical computation or memory use across the two
variants. It also does not imply that the two interfaces contain the same
information: the residual update can carry context that is absent from
the selected source vectors alone. These distinctions are why the
selection-only, selected-context, and unselected-context results are
reported as separate conditions.\label{end:selector_spec}

\FloatBarrier
\apptocmd{\bibsection}{\phantomsection\label{supp:references}}{}{}
\apptocmd{\endthebibliography}{\label{end:bib}}{}{}
\bibliographystyle{abbrvnat}
\bibliography{references}

\end{document}